%% file: main.tex
\documentclass{article}
\usepackage{ijcai21}

\usepackage{times}
\usepackage{soul}
\usepackage{url}
\usepackage[hidelinks]{hyperref}
\usepackage[utf8]{inputenc}
\usepackage[small]{caption}
\usepackage{graphicx}
\usepackage[nosumlimits]{amsmath} 
\usepackage{amssymb} 
\usepackage{amsthm}
\usepackage{booktabs}
\usepackage{cancel}
\usepackage{enumitem} 
\usepackage{framed}
\usepackage{multicol}
\usepackage{stmaryrd} 
\usepackage{thm-restate}
\usepackage[obeyDraft]{todonotes}
\usepackage{xspace}
\usepackage{latexsym}

\def\tr{yes}

\input{macros}

\title{
Signature-Based
Abduction with Fresh Individuals and Complex Concepts for 
Description Logics
\ifdefined\tr
(Extended Version) 
}

\author{
    Patrick Koopmann
    \affiliations
    Institute for Theoretical Computer Science, Technische Universit\"at Dresden
    \emails
    patrick.koopmann@tu-dresden.de
}

\makeatletter
\newcommand{\customLabel}[2]{%
   \protected@write \@auxout {}{\string \newlabel {#1}{{#2}{\thepage}{#2}{#1}{}} }%
   \hypertarget{#1}{}
}
\begin{document}

\maketitle

\todo[inline]{
Decisions:

Observation always $\Phi$
}

\todo[inline]{
Todo later:


- use $\mathfrak{R}$ for role assignments 

- mathsf for all DL names




}

\begin{abstract}
  \input{abstract}
\end{abstract}

\input{introduction}

\input{preliminaries}

\input{flat-solutions}

\input{complex-solutions}

\input{minimal-solutions}

\input{conclusion}

\section*{Acknowledgements}

Patrick Koopmann is supported by DFG grant 389793660 of TRR 248
(see \url{https://www.perspicuous-computing.science/}). 

%

\newpage

\bibliographystyle{named}
\bibliography{bibliography-abduction}

\ifdefined\tr

\newpage


\appendix

\maketitle

\part*{Proof Details}

\input{appendix-flat-solutions}

\input{appendix-complex-solutions-EL}

\input{appendix-complex-solutions-ALC}

\input{appendix-minimal-solutions}

\fi

\end{document}

%% file: macros.tex
\newtheorem{theorem}{Theorem}
\newtheorem{lemma}{Lemma}
\newtheorem{definition}{Definition}

\newcommand{\bigsqcapn}{\bigsqcap\nolimits}
\newcommand{\comment}[1]{}

\newcommand{\tup}[1]{\langle #1 \rangle}
\newcommand{\iinv}[2]{\llbracket #1, #2 \rrbracket}
\newcommand{\liinv}[2]{\left\llbracket #1, #2 \right\rrbracket}
\newcommand{\card}[1]{\# #1}
\newcommand{\powerset}[1]{2^{#1}}
\newcommand{\Naturals}{\ensuremath{\mathbb{N}}\xspace}

\newcommand{\PTime}{\textsc{P}\xspace}
\newcommand{\NP}{\textsc{NP}\xspace}
\newcommand{\ExpTime}{\textsc{ExpTime}\xspace}
\newcommand{\NExpTime}{\textsc{NExpTime}\xspace}
\newcommand{\coNExpTime}{\textsc{coNExpTime}\xspace}
\newcommand{\TwoExpTime}{\textsc{2ExpTime}\xspace}
\newcommand{\NTwoExpTime}{\textsc{N2ExpTime}\xspace}

\newcommand{\NC}{\ensuremath{\mathsf{N_C}}\xspace}
\newcommand{\NR}{\ensuremath{\mathsf{N_R}}\xspace}
\newcommand{\NI}{\ensuremath{\mathsf{N_I}}\xspace}

\newcommand{\Amc}{\ensuremath{\mathcal{A}}\xspace}
\newcommand{\Hmc}{\ensuremath{\mathcal{H}}\xspace}
\newcommand{\Imc}{\ensuremath{\mathcal{I}}\xspace}
\newcommand{\Jmc}{\ensuremath{\mathcal{J}}\xspace}
\newcommand{\Kmc}{\ensuremath{\mathcal{K}}\xspace}
\newcommand{\Lmc}{\ensuremath{\mathcal{L}}\xspace}
\newcommand{\Tmc}{\ensuremath{\mathcal{T}}\xspace}

\newcommand{\EL}{\ensuremath{\mathcal{EL}}\xspace}
\newcommand{\ELbot}{\ensuremath{\EL_\bot}\xspace}
\newcommand{\ALC}{\ensuremath{\mathcal{ALC}}\xspace}
\newcommand{\ALCI}{\ensuremath{\mathcal{ALCI}}\xspace}
\newcommand{\ALCF}{\ensuremath{\mathcal{ALCF}}\xspace}
\newcommand{\ALCQ}{\ensuremath{\mathcal{ALCQ}}\xspace}

\newcommand{\ALCIQ}{\ensuremath{\mathcal{ALCQI}}\xspace}
\newcommand{\ALCQI}{\ensuremath{\mathcal{ALCQI}}\xspace}
\newcommand{\SHIQ}{\ensuremath{\mathcal{SHIQ}}\xspace}
\newcommand{\HornSHIQ}{Horn-\SHIQ\xspace}

\newcommand{\ind}[1]{\textsf{ind}(#1)}
\newcommand{\sub}[1]{\textsf{sub}(#1)}
\newcommand{\sig}[1]{\textsf{sig}(#1)}
\newcommand{\size}[1]{\textsf{size}(#1)}

\newcommand{\AbductionProblem}{\mathfrak{A}}

\newcommand{\blank}{\cancel{b}}

\newcommand{\type}[1]{{\textnormal{tp}(#1)}}
\newcommand{\typeI}[2]{\textnormal{tp}(#1)_{#2}}
\newcommand{\types}[1]{T_{#1}}
\newcommand{\sucCand}[3]{\textnormal{S}_{#1}^{#3}(#2)}

\newcommand{\Bit}{\textsf{Bit}}
\newcommand{\NBit}{\textsf{NBit}}
\newcommand{\Flip}{\textsf{Flip}}
\newcommand{\Init}{\textsf{Init}}
\newcommand{\Error}{\textsf{Error}}
\newcommand{\Goal}{\textsf{Goal}}

%% file: abstract.tex
Given a knowledge base and an observation as a set of facts, 
ABox abduction aims at computing a \emph{hypothesis} that, when added to the 
knowledge base, is sufficient to entail the observation. In signature-based 
ABox abduction, the hypothesis is further required to use only names from a 
given set. This form of abduction has applications such as diagnosis, KB 
repair, or explaining missing entailments. It is possible that hypotheses for 
a given observation only exist if we admit the use of fresh individuals
and/or complex concepts built from the given signature, something most 
approaches %
so far do not support or only support
with restrictions. In this paper, we investigate the computational complexity 
of this form of abduction---allowing either fresh individuals, complex concepts, 
or both---for various description logics, and give size bounds on the hypotheses 
if they exist.

%% file: introduction.tex
\section{Introduction}

Description logics (DLs) are a powerful formalism to describe knowledge bases 
(KBs) containing both general domain knowledge from a DL ontology and a set of 
facts (the ABox). Using a DL reasoner, we can then infer information that is 
implicit in the data, and can be logically deduced based on the 
ontology \cite{DL_TEXTBOOK}. Sometimes it is useful to not only reason about 
what logically follows from a DL KB, but to reason also about what does 
\emph{not} follow. 
In abduction, we are given a KB as background knowledge, in combination with a 
set of facts (the observation) that cannot be deduced from the background 
knowledge. We are then looking for the missing piece in the background knowledge 
(the hypothesis) that is needed to make the observation logically 
entailed \cite{FOUNDATIONS_DL_ABDUCTION}. This form of reasoning has many 
applications: 1) it 
can be used to \emph{explain} why something cannot be 
deduced \cite{DL_LITE_ABDUCTION}, 2) it can be used for \emph{diagnosis} tasks, 
giving the hypothesis as possible explanation for an unexpected
observation \cite{ABoxAbductionMedicalDiagnosis}, and 3) it can be used in 
\emph{KB repair} to give hints on how to fix 
missing entailments \cite{WeiKleinerDragisicLambrix2014}.

\newcommand{\Unstable}{\textsf{U}}
\newcommand{\GypsumFormation}{\textsf{E}}
\newcommand{\borders}{\textsf{b}}
\newcommand{\protectiveBarrier}{\textsf{b}}
\newcommand{\ab}{\textsf{ab}}

As a simplified application example from the geology domain, assume we have 
observed that in an area near a canal, holes appeared in the street as a result 
of subsidence due to an unstable ground. A possible explanation could involve 
the presence of a formation of so-called \emph{evaporite} below the street, 
which dissolves when in contact with water \cite{EVAPORITE}. Our background 
knowledge consists of 
a geology ontology together with data about the area. Among others, 
it contains the following abbreviated axioms:
 \begin{align*}
  &1.\ \text{EvaFor} \sqcap \exists \text{bord}.(\text{Wat}
  \sqcap\neg\exists\text{lin}.\text{WatPro}) \sqsubseteq 
  \exists \text{aff}.\text{Dis}\\
  &2.\ \text{EvaFor} \sqcap \exists \text{aff}.\text{Dis} 
  \sqsubseteq \forall \text{abov}.\text{Unst}\\
 &3.\ (\text{Wat}\sqcup\text{Str}) \sqcap \text{EvaFo}\sqsubseteq\bot \qquad
 4.\ \text{Wat}(\textit{can}) \qquad
 5.\ \text{Str}(\textit{str})
  \end{align*}
which state that 1.~an \textbf{Eva}porite \textbf{For}mation %
which \textbf{bo}rders to a 
\textbf{Wat}erway 
without \textbf{Wat}er-\textbf{Pro}of \textbf{lin}ing
will be \textbf{aff}ected by \textbf{Dis}solution; 2. All ground 
\textbf{abov}e 
an evaporite formation affected by dissolution is \textbf{Uns}table; 
3. waterways and \textbf{Str}eets are not evaporite formations; 
4. $\textit{can}$ is a waterway; 5. $\textit{str}$ is a 
street. 
Our observation would be that the street is unstable: 
$\text{Unst}(a_1)$, and a hypothesis based on our background knowledge would be
\begin{gather*}
 \Hmc = \{\ \text{EvaFor}(e),\ \text{abov}(e,\textit{str}),\ 
 \text{bord}(e,\textit{can}),\ 
 \forall\text{lin}.\bot(\textit{can}) \ \}
\end{gather*}
stating that there is an evaporite formation $e$ below the street that 
borders with the canal, and that the canal has no lining. A team of geologists 
can then verify the hypothesis by analysing the canal and the ground below the 
street. We highlight two aspects of this hypothesis: 1) it refers to a 
previously 
unknown individual, the evaporite formation, and 2) it uses a complex 
(composed) DL concept ($\forall\text{lin}.\bot$). We are interested 
in hypotheses like this for \emph{signature based abduction}, where we are 
additionally given a 
signature $\Sigma$ of \emph{abducibles}---a vocabulary of names to be used 
within the hypothesis \cite{OUR_KB_ABDUCTION}. The aim of $\Sigma$ is to 
restrict to hypotheses that have explanatory character. In 
the present example, we would exclude $\textit{aff}$ and $\textit{Dis}$ 
from~$\Sigma$, as the dissolution alone would be a too shallow explanation, and 
$\textit{Wat}$, because we already know the waterways in the area.
Furthermore, we are looking 
at \emph{ABox abduction}, in which %
observations and hypotheses are ABoxes, 
in contrast to TBox 
abduction \cite{du2017practical,WeiKleinerDragisicLambrix2014}, KB abduction
\cite{OUR_KB_ABDUCTION,FOUNDATIONS_DL_ABDUCTION} or concept 
abduction \cite{bienvenu2008complexity}.

While there are practical approaches to ABox abduction without 
signature-restriction \cite{KlarmanABox2011,HallandBritzABox2012,%
PukancovaHomola2017}, works on signature-based ABox abduction often 
restrict hypotheses to flat ABoxes with a given set of 
individuals \cite{EX_RULES_ABDUCTION,DuQiShenPan2012}---which essentially means 
that statements in a hypothesis can be picked from a finite set---or 
they restrict to \emph{rewritable} DLs which have limited 
expressivity \cite{DuWangShen2014,DL_LITE_ABDUCTION}.
As with DLs, we usually have the open-world semantics, in which 
not all individuals are known, and DLs offer much more expressivity, it is a 
natural next step to look also at abduction allowing for 
fresh individuals and complex concepts in the result. This changes the nature 
of the abduction problem drastically as there 
is now an unbounded set of axioms that may occur in a hypothesis. Problems 
such as ``\emph{Does axiom 
$\alpha$ belong to some/every/an optimal 
solution?}'' \cite{DL_LITE_ABDUCTION,EX_RULES_ABDUCTION} %
become
less helpful while new questions become interesting, such as whether we can 
give bounds on the number of individuals in a hypothesis, or on the 
overall size of the hypothesis.

Without understanding the theoretical properties yet, practical methods for 
signature-based abduction that admit expressive DL concepts in the hypothesis 
are presented in \cite{OUR_KB_ABDUCTION,WARREN_ABOX_ABDUCTION}.
The authors there consider hypotheses that we would call
\emph{complete} in the sense that they cover all hypotheses at the 
same time. To make this possible, solutions are represented in a very 
expressive DL using non-classical operators 
such as fixpoints and axiom disjunction. In this setting, ABox abduction can be 
reduced to uniform interpolation, which however may produce solutions that are 
triple exponentially large 
\cite{FORGETTING-ZHAO,FOUNDATIONS_UI}. A natural question is 
whether this blow-up is really necessary, or whether we can obtain smaller or 
simpler hypotheses if we drop the requirement of completeness and look for 
hypotheses in a classical DL
that is sufficient to entail the 
observation. 

Unfortunately, our results indicate that this is not the case: 
if we only allow for fresh individuals but not for complex concepts, 
hypotheses may require exponentially many assertions
for DLs between $\ELbot$ and $\ALCI$, 
while for $\ALCF$ there does not even exist a general bound. If in addition, 
we allow for complex concepts, we are able to explain more observations, but 
the explanations may become \emph{triple exponentially large} in comparison to 
the input. Motivated by this, we also consider a variant of the 
abduction problem in which we are additionally given a bound on the size of the
hypothesis.
To summarize, our contributions are the following.
\begin{enumerate}
 \item we investigate signature-based ABox abduction for DLs ranging from $\EL$ to $\ALCIQ$ where 
hypotheses may use \emph{fresh individuals}, \emph{complex concepts} or \emph{both},
 \item we give tight bounds on the size of hypotheses if they exist, 
 \item we analyse the computational complexity of deciding whether a hypothesis 
 exists, and
 \item we analyse the complexity of deciding whether a hypothesis of bounded 
size exists.
\end{enumerate}

\ifdefined\tr
    Proof details are provided in the appendix. 
\else
    Proof details are provided in the technical report~\cite{TR}.
\fi

%% file: preliminaries.tex
\section{Description Logics and ABox Abduction}

We recall the DLs relevant to this paper~\cite{DL_TEXTBOOK} and provide the 
formal definition of the abduction problem  we consider in this paper.

Let $\NC$, $\NR$ and $\NI$ be three pair-wise disjoint sets of respectively 
\emph{concept}, \emph{role} and \emph{individual names}. A \emph{role} $R$ is 
either a role name $r$ or an \emph{inverse role} $r^-$, where $r\in\NR$. 
\emph{$\EL$ concepts} are built according to the following syntax rule, where 
$A\in\NC$ and $R\in\NR$:
\[
 C ::= \top\mid A\mid C \sqcap C \mid \exists R.C 
\]
More expressive DLs allow for the following additional concepts, 
where $n\in\Naturals$, and in brackets, we give the name of the corresponding 
DL:
\begin{align*}
       \bot \ (\ELbot) \ \
       \neg C \ (\ALC)\ \
       {\leqslant}1R.\top (\ALCF)\ \
       {\leqslant}nR.C (\ALCQ)
\end{align*}
In 
each case, all previous constructs are allowed in the DL as well. Using the 
letter $\Imc$ in the DL name we express that in the above, $R$ may also be an 
inverse role. 
For 
example, $({\geq}n r^-.C)$ is an $\ALCIQ$ concept but not an 
$\ALCQ$-concept, and $\exists r.\bot$ is an $\ELbot$ concept. 
Additional 
operators are introduced as abbreviations: $C\sqcup D=\neg(\neg C\sqcap\neg 
D)$ and $\forall R.C=\neg\exists R.\neg C$.

A \emph{KB} is a set of \emph{axioms}, that is, \emph{concept inclusions} (CIs) 
$C\sqsubseteq D$, \emph{concept assertions} $C(a)$ and \emph{role assertions} 
$r(a,b)$, where $C,D$ are concepts, $a,b\in\NI$ and $r\in\NR$. If a KB contains 
only concept and role assertions, it is called \emph{ABox}, and if every concept 
assertion is of the form $A(a)$, where $A\in\NC$, \emph{flat ABox}. Given a 
concept/axiom/KB/ABox $E$, we denote by $\sub{E}$ the set of (sub-)concepts 
occurring 
in~$E$, by $\sig{E}$ the set of concept and role names occuring in $E$, and by 
$\ind{E}$ the set of individual names in $E$. By $\size{E}$, we denote the 
number of symbols required to write $E$ down, where operators, as well 
as concept, role and individual names count as one, numbers are encoded in 
binary and the introduced abbreviations can be used.

The semantics of DLs is defined based on \emph{interpretations}, which are 
tuples $\Imc=\tup{\Delta^\Imc,\cdot^\Imc}$ of a set $\Delta^\Imc$ of 
\emph{domain 
elements} and an \emph{interpretation function} $\cdot^\Imc$ which maps every 
$a\in\NI$ to some $a^\Imc\in\Delta^\Imc$, every $A\in\NC$ to some
$A^\Imc\subseteq\Delta^\Imc$, every $r\in\NR$ to some
$r^\Imc\subseteq\Delta^\Imc\times\Delta^\Imc$, satisfies 
$(r^-)^\Imc=(r^\Imc)^-$, and is extended to concepts as follows, where 
$\card{S}$ denotes the cardinality of the set $S$:
\begin{gather*}
 \bot^\Imc=\emptyset\quad 
 (C\sqcap D)^\Imc=C^\Imc\cap D^\Imc \quad
 (\neg C)^\Imc=\Delta^\Imc\setminus C^\Imc\\
 (\exists R.C)^\Imc=\{d\in\Delta^\Imc\mid \tup{d,e}\in R^\Imc,\ e\in 
C^\Imc\}\\
({\leqslant}nR.C)^\Imc=\{d\in\Delta^\Imc\mid \#\{\tup{d,e}\in R^\Imc\mid e\in 
C^\Imc\} \leq n \}
\end{gather*}
$\Imc$ \emph{satisfies} an axiom $\alpha$, in symbols $\Imc\models\alpha$, if 
$\alpha=C\sqsubseteq D$ and $C^\Imc\subseteq D^\Imc$; $\alpha=C(a)$ and 
$a^\Imc\in C^\Imc$; or $\alpha=r(a,b)$ and $\tup{a^\Imc,b^\Imc}\in r^\Imc$. If 
$\Imc$ satisfies all axioms in a KB $\Kmc$, we write $\Imc\models\Kmc$ and call 
$\Imc$ a \emph{model of $\Kmc$}. A KB \emph{entails} an axiom/KB~$E$, in 
symbols 
$\Kmc\models E$, if $\Imc\models E$ for every model~$\Imc$ of $\Kmc$. If $\Kmc$ 
has no model, it is \emph{inconsistent} and we write $\Kmc\models\bot$.

We can now define the main reasoning problem we are concerned with in this 
paper.
We consider \emph{signature-based ABox abduction problems}, which for 
convenience, we 
just call abduction problems from here on.
\begin{definition}
 Let $\Lmc$ be a DL. An \emph{$\Lmc$ abduction problem} is given by a triple 
$\AbductionProblem=\tup{\Kmc,\Phi,\Sigma}$, with an $\Lmc$ KB $\Kmc$ of 
\emph{background knowledge}, an 
$\Lmc$ ABox $\Phi$ as \emph{observation}, and a signature 
$\Sigma\subseteq\NC\cup\NR$
of \emph{abducibles}; and asks whether there exists a \emph{hypothesis for 
$\AbductionProblem$}, i.e. an $\Lmc$ ABox $\Hmc$ satisfying 
 \customLabel{itm:consistency}{\textbf{A1}}
 \customLabel{itm:entailment}{\textbf{A2}}
 \customLabel{itm:signature}{\textbf{A3}}
\begin{align*}
 \textbf{A1}\ \ \Kmc\cup\Hmc\not\models\bot, \ \ \ \ 
 \textbf{A2}\ \ \Kmc\cup\Hmc\models\Phi,\text{ and} \ \ \  
 \textbf{A3}\ \ \sig{\Hmc}\subseteq\Sigma.
\end{align*}%
If we require $\Hmc$ additionally to be flat, we speak of a \emph{flat 
abduction problem}. 
\end{definition}
Condition~\ref{itm:consistency} is required to avoid trivial solutions, since 
everything follows from a basic inconsistency. Condition~\ref{itm:entailment} 
ensures the hypothesis is indeed effective, and Condition~\ref{itm:signature} 
is what makes this a \emph{signature-based} abduction problem.

In addition, different minimality criteria can be put on the hypothesis, such 
as \emph{size-}, \emph{subset-}, \emph{semantic 
minimality}~\cite{DL_LITE_ABDUCTION}, 
or \emph{completeness}~\cite{OUR_KB_ABDUCTION}. In this paper, we consider 
size minimality, for which the corresponding decision problem is
the following.
\begin{definition}\label{def:abduction}
 A \emph{size-restricted (flat) $\Lmc$ abduction problem} is a tuple 
$\AbductionProblem=\tup{\Kmc,\Phi,\Sigma,n}$, where 
$\AbductionProblem'=\tup{\Kmc,\Phi,\Sigma}$ is a (flat)
$\Lmc$ abduction problem and $n$ is a number encoded in binary. A 
\emph{hypothesis for $\AbductionProblem$} is an $\Lmc$-ABox $\Hmc$ which is a 
hypothesis for~$\AbductionProblem'$ and additionally satisfies 
\mbox{$\size{\Hmc}\leq n$}.
\end{definition}

%% file: flat-solutions.tex
\section{Flat ABox Abduction}

We first consider flat ABox abduction: the size of hypotheses, and the 
complexity of deciding their existence.
This problem is similar to that of \emph{query 
emptiness}~\cite{QUERY_EMPTINESS} which asks whether for a given set $\Tmc$ of 
CIs, Boolean query $q$ and signature $\Sigma$, there exists a flat ABox $\Amc$ 
with 
$\sig{\Amc}\subseteq\Sigma$, s.t. $\Tmc\cup\Amc$ entails that query.
Query emptiness for 
\emph{instance queries} is essentially 
the special case of flat ABox abduction where the observation is of the form 
$A(a)$ and the background knowledge contains only CIs. This immediately gives 
us lower bounds for the decision problem of flat ABox abduction. The other 
results in this section do not follow, but can be shown similarly as 
in~\cite{QUERY_EMPTINESS}.
Similar to query emptiness, flat ABox abduction only becomes interesting if 
the DL is powerful enough to create inconsistencies.
Otherwise, we can construct a trivial hypothesis candidate as 
\begin{align*}
    \Hmc=&\{A(a)\mid A\in(\Sigma\cap\NC),a\in\ind{\Kmc\cup\Phi}\}\\
    \cup&\{r(a,b)\mid r\in(\Sigma\cap\NR),a,b\in\ind{\Kmc\cup\Phi}\}.
\end{align*}
Clearly, $\Hmc$ satisfies~\ref{itm:consistency} and~\ref{itm:signature}. If it 
does not satisfy~\ref{itm:entailment}, then neither would any other ABox. 
Consequently, if there is a solution to the abduction problem, then $\Hmc$ must 
be such a solution. This means that for \EL, 
flat ABox abduction can always be 
performed in polynomial time.
On the other hand, already for $\ELbot$, solutions may require exponentially 
many fresh individual names.
\begin{theorem}\label{the:exp-lower-bound-flat}
 There exists a family $\tup{\Kmc_n,\Phi,\Sigma}_{n>0}$ of flat 
$\ELbot$~abduction problems
s.t. every hypothesis is of size exponential in the size of $\Kmc_n$.
\end{theorem}
\begin{proof}
 We set $\Phi=A(a)$ and $\Sigma=\{B,r\}$, and let $\Kmc_n$ use concept names 
$X_1$, 
$\overline{X}_1$, $\ldots$, $X_n$, $\overline{X}_n$ to encode an exponential 
counter.  $\Kmc_n$ uses this to ensure that an $r$-chain of distinct $2^n$ 
elements from $a$ to an instance of $B$ entails $A(a)$.
 \begin{align*}
  B & \sqsubseteq \overline{X}_1\sqcap\ldots\sqcap\overline{X}_n \\
  \exists r.(\overline{X}_i\sqcap X_{i-1}\sqcap\ldots X_1) &\sqsubseteq X_i 
  \quad \text{ for }i\in\iinv{1}{n}\\
  \exists r.(X_i\sqcap X_{i-1}\sqcap\ldots\sqcap X_1) &\sqsubseteq \overline{X}_i 
  \quad \text{ for }i\in\iinv{1}{n}\\
  \exists r.(\overline{X_i}\sqcap\overline{X}_j) &\sqsubseteq \overline{X}_i 
  \quad \text{ for }i,j\in\iinv{1}{n}, j<i \\
  \exists r.(X_i\sqcap\overline{X}_j) &\sqsubseteq X_i 
  \quad \text{ for }i,j\in\iinv{1}{n}, j<i \\
   X_i\sqcap \overline{X_i}&\sqsubseteq\bot\quad\ \  
   \text{ for }i\in\iinv{1}{n}\\
  X_1\sqcap\ldots\sqcap X_n &\sqsubseteq A %
 \end{align*}
 The only way to produce this chain of $2^n$ elements is using $2^{n}-1$ 
role assertions, which establishes our lower bound.
\end{proof}
This bound remains tight if we add expressivity up to $\ALCI$, while we lose 
any bound on the size once we additionally allow concepts of the form 
${\leqslant}1r.\top$.  
\begin{restatable}{theorem}{TheFlatSizeUpper}\label{the:flat-size-upper}
 If there exists a hypothesis for a flat $\Lmc$ abduction problem 
$\AbductionProblem$,  then there exists one of size 
\begin{enumerate}
 \item polynomial in the size of $\AbductionProblem$ if $\Lmc=\EL$, 
 \item exponential in the size of $\AbductionProblem$ if $\Lmc=\ALCI$, and
 \item if $\Lmc=\ALCF$, no general upper bound based on $\AbductionProblem$ can 
be given.
\end{enumerate}
\end{restatable}
\begin{proof}[Proof sketch]
We already established the bound for $\Lmc=\EL$. For $\ALCI$, we assume there exists some 
hypothesis 
$\Hmc_0$, based on which we build one of bounded size. For this, we pick any 
model 
$\Imc$ of $\Hmc_0\cup\Kmc$, which allows us to identify individual names $a\in\ind{\Hmc_0}$ 
using at most exponentially many \emph{types} $\type{a}=\typeI{a^\Imc}{\Imc}$, 
where 
\[
    \typeI{d}{\Imc}=\{C\in\sub{\Kmc\cup\Phi}\mid d\in C^\Imc\}%
    .
\] 
We associate to every type $\type{a}$ an individual name 
$b_{\type{a}}$, and define 
$ 
h:\ind{\Hmc_0}\rightarrow\NI
$
by $h(a)=a$ if \mbox{$a\in\ind{\Kmc\cup\Phi}$} and $h(a)=b_{\type{a}}$ 
otherwise. The hypothesis $\Hmc$ is then: %
\begin{align*}
 \{A(h(a))\mid A(a)\in\Hmc_0\}\cup\{r(h(a),h(b))\mid r(a,b)\in\Hmc_0\}.
\end{align*}
Based on $\Imc$ and $h$, one can construct a model for $\Kmc\cup\Hmc$, so that $\Hmc$ 
satisfies~\ref{itm:consistency}. Because $\Hmc_0$ satisfies~\ref{itm:signature}, by 
construction, so does $\Hmc$. 
Finally, using the fact that $h$ is a homomorphism from $\Kmc\cup\Hmc_0$ into 
$\Kmc\cup\Hmc$ s.t. for every $a\in\ind{\Kmc\cup\Phi}$ 
$h(a)=a$, we can show that $\Kmc\cup\Hmc\models\Phi$, and 
thus~\ref{itm:entailment}.

For $\Lmc=\ALCF$, we note that if there was a bound on the size of hypotheses, 
we could decide the instance query emptiness problem for $\ALCF$ by iterating 
over all candidates, contradicting that this problem is undecidable for 
$\ALCF$~\cite{QUERY_EMPTINESS}.
\end{proof}

The proof of Theorem~\ref{the:flat-size-upper} indicates how types can be used 
to perform abduction, which is used in the following theorem.

\begin{restatable}{theorem}{TheFlatCompute}\label{the:compute-flat}
 Flat $\Lmc$ ABox abduction is
 \begin{itemize}
  \item \PTime-complete for $\Lmc=\EL$,
  \item \ExpTime-complete for $\Lmc=\ELbot$,
  \item \coNExpTime-complete for $\Lmc=\ALCI$,
  \item undecidable for $\Lmc=\ALCF$.
 \end{itemize}
\end{restatable}
\begin{proof}[Proof sketch]
 For $\EL$, we already described a method to compute and verify hypotheses in \PTime. For the other 
DLs, we 
use type elimination to compute in exponential time the set $T_{\Kmc\cup\Phi}$ 
of types $t\subseteq\sub{\Kmc\cup\Phi}$ that can occur in models of $\Kmc$. For 
a type $t$ and a set $T$ of types, we denote by
$\sucCand{T}{t}{r}\subseteq T$ the types an $r$-succesor of a 
$t$-individual can have.
We iterate over all assignments 
$s:\ind{\Kmc\cup\Phi}\rightarrow T_{\Kmc\cup\Psi}$ for which we set
\begin{align*}
 \Hmc_s =& \{A(a)\mid a\in I, A\in s(a)\cap\Sigma\}\\
        &\cup \{r(a,b)\mid a,b\in I, r\in\Sigma, 
s(b)\in\sucCand{T_{\Kmc\cup\Phi}}{s(a)}{r}\},
\end{align*}
where $I$ contains an additional individual name for every type and $s$ is 
extended to $I$.
If there exists a hypothesis, then the hypothesis constructed in the 
proof for Theorem~\ref{the:flat-size-upper} is contained in some such $\Hmc_s$ 
s.t. $\Kmc\cup\Hmc_s\not\models\bot$, so that we can directly 
check~\ref{itm:consistency}--\ref{itm:signature} from Def.~\ref{def:abduction} 
on the generated $\Hmc_s$.
\ref{itm:consistency} 
can be verified syntactically using the types. For $\ELbot$, 
\ref{itm:entailment} is 
verified in time polynomial in the size of $\Hmc_s$ and thus in \ExpTime. For 
$\ALCI$, we guess a model of $\Kmc\cup\Hmc_s$ that does not entail $\Phi$ 
by assigning a type to every individual and possibly adding a new 
element per type. If this fails, $\Kmc\cup\Hmc_s\models\Phi$.
\end{proof}

%% file: complex-solutions.tex
\section{Abduction with Complex Concepts}

As illustrated in the introduction, abduction may only be successful if we also admit 
complex concepts in the hypothesis. Determining such hypotheses turns out to be 
more challenging than for flat hypotheses, and we cannot find a correspondence 
to a known problem as for flat abduction. Indeed, one might assume such a 
relation to uniform interpolation: given a KB $\Kmc$ and a 
signature $\Sigma$, the uniform interpolant of $\Kmc$ for $\Sigma$ is a
$\Sigma$ ontology that captures all entailments of $\Kmc$ within 
$\Sigma$~\cite{ABOX-FORGETTING}. By negating the observation, this can be used 
to perform \emph{complete} 
abduction~\cite{OUR_KB_ABDUCTION,WARREN_ABOX_ABDUCTION}, that is, to compute a 
hypothesis that would be entailed by any other hypothesis. However, if we are 
interested just in computing any hypothesis rather than a complete one,
this correspondence falls short, as uniform interpolants have stronger 
requirements than hypotheses, and the reasons for non-existence are 
different: namely, capturing \emph{all} entailments of $\Kmc$ in 
$\Sigma$ in the uniform interpolant, using only names from $\Sigma$, 
may require infinitely many axioms in case of cyclic axioms. In contrast, for abduction, 
non-existence is always due to Condition~\ref{itm:consistency}.

We consider abduction for $\ALC$, $\EL$, and $\ELbot$, starting with the 
latter. 
In $\EL$ and $\ELbot$, complex concepts do 
not bring much benefit compared to fresh individuals:
an \ELbot concept can only 
state the existence of role successors, which we can also do in flat ABoxes.
In fact, for $\ELbot$, if we 
allow complex concepts \emph{instead} of fresh individuals , hypotheses even 
get more complex.

\begin{theorem}\label{the:el-complex}
 There exists a family of \ELbot abduction problems for which every 
hypothesis without fresh individuals is at least of double exponential size. 
If there 
exists such a hypothesis, there always exists one of at most double 
exponential size, whose existence can be decided in exponential time.
\end{theorem}
\begin{proof}[Proof sketch]
 The family of abduction problems is obtained similarly as in the proof 
for Theorem~\ref{the:exp-lower-bound-flat}, only that we now use two roles $r$ 
and $s$
To get 
the corresponding upper bound, we first flatten an 
existing hypothesis $\Hmc_0$ and again simplify the ABox based on the types 
in some model, however this time making sure the resulting flat ABox can 
be translated back into a complex one without fresh individuals.
 
The same care has to be taken when we modify the method 
used for Theorem~\ref{the:compute-flat}. 
Specifically, we have to make sure that the hypothesis $\Hmc_s$ that we 
generate for a given mapping $s:\ind{\Kmc\cup\Phi}\rightarrow T_{\Kmc\cup\Phi}$
does not contain cycles 
between fresh individuals, so that it can be translated into a hypothesis 
without fresh individuals.  Our fresh individuals are now of the form $b_{a,t,i}$, 
where 
$a\in\ind{\Kmc\cup\Phi}$, $t\in T=T_{\Kmc\cup\Phi}$, and 
$i\in\iinv{1}{2^{\lvert 
T\rvert}}$. 
Set $b_{a,s(a),0}=a$. $\Hmc_s$ is then defined as:
\let\qed\relax
\begin{align*}
 \Hmc_s& = \big\{A(b_{a,t,k})\mid a\in\ind{\AbductionProblem}, 
                             t\in T, 
                             k\in\iinv{0}{2^{\lvert T\rvert}},
                             \phantom{\big\}}\quad
                             \\ 
&\hphantom{ = \big\{A(b_{a,t,k})\mid %
                             .}
                             A\in (T\cap\Sigma)\big\}
      \\
  &\cup\ \big\{r(b_{a,t_1,k}, b_{a,t_2,k+1}) \mid 
                             a\in\ind{\Kmc\cup\Phi},
                             t_1\in T, 
                             r\in\Sigma,
                             \phantom{\big\}}
                             \\
    &\hphantom{\cup\ \{r(b_{a,t_1,k}, b_{a,t_2,k+1}) \mid ..}
                             t_2\in\sucCand{T}{t_1}{r},
                             k\in\iinv{0}{2^{\lvert T\rvert}-1}\big\}\\
  &\cup\ \big\{r(a,b)\mid s(b)\in\sucCand{T}{s(a)}{r}\big\}
  \hphantom{aaaaaaaaaaaaaaaaa..}\qedsymbol
\end{align*}
\end{proof}

The hardness result requires again $\bot$: for \EL, we can always use a flat 
solution as in the last section.
In contrast, with more expressivity, the problem becomes even harder,
even if we 
do admit fresh individuals. The reason is that for concepts of the form 
$\forall r.C$, fresh individuals cannot come to the rescue anymore, and 
disjunctions may become necessary.
The following theorem is shown by a modification of 
a construction in~\cite{CONSERVATIVE_EXTENSIONS}.

 \begin{restatable}{theorem}{TheLowerComplexALC}
  	There is a family of $\ALC$ abduction problems for which 
    the smallest (non-flat) ABox hypotheses are triple exponential in size.
 \end{restatable}
 We can show that this bound is tight.
 \begin{theorem}\label{the:upper-alc-complex}
  Let $\AbductionProblem$ be an $\ALC$ abduction problem. Then, there exists a 
hypothesis for $\AbductionProblem$ iff there exists a hypothesis 
of triple exponential size. 
 \end{theorem}
 To show this theorem, we use a technique similar as for 
Theorem~\ref{the:flat-size-upper}: 
we  
take an arbitrary hypothesis, and transform it into one of triple 
exponential size. However, this time, a construction based on a single model is 
not sufficient, and we have to take into account an appropriate abstraction of 
\emph{several} models of $\Kmc\cup\Hmc_0$. %
We thus proceed as follows:
\begin{enumerate}
 \item we abstract the KB $\Kmc\cup\Hmc_0$ into a model abstraction,
 \item we reduce the size of this abstraction, 
 \item based on which we construct a hypothesis $\Hmc$ of 
 triple exponential  
       size.
\end{enumerate}

In the model abstraction, elements are represented as nodes $v\in V$ that are 
labeled with a set $\lambda(v)$ of types with the intuitive meaning ``this 
element may have one of the types in $\lambda(v)$''. Role relations are 
represented using tuples $\tup{v_1,t,r,v_2}$ which are read as: if the node 
$v_1$ has type $t$, then it has an $r$-successor corresponding to $v_2$. 
Roughly, from a 
model abstraction, we can obtain a model using the following inductive 
procedure: 1) start with the nodes that represent individuals, 2) assign to 
each node a type from its label set, 3) if for those types, the node requires 
successor nodes, add those and continue in 2). To allow for unbounded paths in models for 
finite acyclic model abstractions,
we further have ``open'' nodes whose role successors are only restricted by the TBox.

\newcommand{\IntAbs}{\mathfrak{I}}
\newcommand{\Rf}{\mathfrak{R}}

\begin{definition}\label{def:interpretation-abstraction}
An \emph{interpretation abstraction} for $\tup{\Kmc,\Phi,\Sigma}$ is a 
tuple $\IntAbs=\tup{V, \lambda, s, \Rf,F}$, where 
\begin{itemize}
    \item $V$ is a set of \emph{nodes}, 
    \item $\lambda:V\rightarrow\powerset{\types{\Kmc\cup\Phi}}$ maps each node 
to 
a set of types, 
    \item partial function $s:\NI\nrightarrow V$ assigns individuals to nodes
    \item 
$\Rf\subseteq \big(V\times\types{\Kmc\cup\Phi}\times(\Sigma\cap\NR)\times 
V\big)$ 
is the 
\emph{role assignment}, 
\item and $F\subseteq V$ is the set of \emph{open nodes}. 
\end{itemize}
$\IntAbs$ \emph{abstracts an interpretation $\Imc$} if there is a subset 
$\Delta'\subseteq\Delta^\Imc$ and a 
function $h:\Delta'\rightarrow V$ s.t. 
for every $d\in\Delta'$ 
and $r\in(\Sigma\cap\NR)$:
 \begin{enumerate}[label=\textbf{I\arabic*}]
  \item\label{i:individuals} for all $a\in\NI$, if $s(a)$ is defined, then
$s(a)=h(a^\Imc)$ 
  \item\label{i:type} $\typeI{d}{\Imc}\in \lambda(h(d))$,
  \item\label{i:roles} if $h(d)\not\in F$, then for every $e\in\Delta^\Imc$, 
$\tup{d,e}\in 
r^\Imc$ iff $e\in\Delta'$ and $\tup{h(d),\typeI{d}{\Imc},r,h(e)}\in\Rf$.
\end{enumerate}
\end{definition}

We need some additional requirements to make sure an interpretation 
abstraction can be represented as an \ALC ABox $\Hmc$ s.t. 
$\sig{\Hmc}\subseteq\Sigma$. 
We call a node $v\in V$ for which there exists $a\in\NI$ with $s(a)=v$ 
\emph{internal node}, and otherwise \emph{outgoing node}. If $v=s(a)$ for 
$a\in\ind{\Kmc\cup\Phi}$, we call $v$ \emph{named node}. 
A 
\emph{path in $\tup{V,\lambda,s,\Rf,F}$} is a sequence 
\[
    \pi=v_0,t_0,r_0,v_1,t_1,r_1,\ldots t_{n-1}r_{n-1}v_n
\]
s.t. for each 
$i\in\iinv{0}{n-1}$, $\tup{v_i,t_i,r_i,t_{i+1}}\in \Rf$. $\pi$ is cyclic if 
it contains a node twice, and its \emph{length} is its number of nodes.

\begin{definition}
$\IntAbs=\tup{V,\lambda,s,\Rf,F}$ is called \emph{$\ALC$-conform} if 
\begin{enumerate}[label=\textbf{D\arabic*},series=conformSig]   
 \item\label{d:acyclic} there is no cyclic path between outgoing nodes  ,
 \item\label{d:unique-roles} for every internal node $v$, 
if $\tup{v,t,r,v'}\in \Rf$, then $\tup{v,t',r,v'}\in \Rf$ for every 
    $t'\in\lambda(v')$, and
 \item\label{d:universal} for every 
$\tup{v_1,t,r,v_2}\in \Rf$, where $v_2$ is internal, there exists 
$\tup{v_1,t,r,v_2'}\in \Rf$ s.t. $v_2'$ is outgoing. 
\end{enumerate}
We say that $\IntAbs$ is \emph{$\Sigma$-complete} 
if
\begin{enumerate}[resume*=conformSig]
    \item\label{d:signature-nodes} for every $v\in V$, and 
        $t_1\in\lambda(v)$, $\lambda(v)$ contains every type 
        $t_2\in\types{\Kmc\cup\Phi}$ s.t. 
        $\lambda(t_1)\cap\Sigma=\lambda(t_2)\cap\Sigma$, and 
    \item\label{d:signature-edges} for every $\tup{v_1,t,r,v_2}\in \Rf$ and 
$t'\in\types{\Kmc\cup\Phi}$ s.t. $t\cap\Sigma=t'\cap\Sigma$, also
$\tup{v_1,t',r,v_2}\in \Rf$.
\end{enumerate}
\end{definition}

\ref{d:acyclic} ensures that we can represent the outgoing paths from a 
node $s(a)$ in a single assertion $C(a)$. \ref{d:unique-roles} 
expresses that the relations between internal nodes is independent of the type, 
which allows to represent them using role assertions. %
\ref{d:universal} is needed to capture allowed paths using universal role 
restrictions.
~\ref{d:signature-nodes} 
and~\ref{d:signature-edges} ensure
that $\IntAbs$ can be captured using only names in $\Sigma$.

\comment{
\begin{restatable}{lemma}{LemABoxToIntAbs}\label{lem:abox-to-int-abs}
 For every $\ALC$ ABox $\Amc$, there exists an $\ALC$-conform, 
$\sig{\Amc}$-complete interpretation abstraction $\IntAbs$ s.t. for every model 
$\Imc$ of $\Kmc$, $\IntAbs$ abstracts $\Imc$ iff $\Imc\models\Amc$.
\end{restatable}
}
From here on, we fix an abduction problem 
$\AbductionProblem=\tup{\Kmc,\Phi,\Sigma}$.
We say that $\IntAbs$ \emph{explains $\Phi$} 
iff some model of $\Kmc$ is abstracted by $\IntAbs$ and every model of $\Kmc$ 
that is abstracted by $\IntAbs$ entails $\Phi$. If $\IntAbs$ is $\ALC$-conform, 
$\Sigma$-complete, and explains $\Phi$, we call it a \emph{hypothesis 
abstraction}.

\begin{restatable}{lemma}{LemABoxToIntAbs}\label{lem:abox-to-int-abs}
 Every hypothesis $\Hmc$ for $\AbductionProblem$ can be translated into a 
hypothesis 
abstraction.
\end{restatable}

Now that we can translate hypotheses into an alternative representation, the 
next step is to decrease their size while making sure they still explain the 
observation. For this, we can now use similar techniques as in 
Theorems~\ref{the:flat-size-upper} and~\ref{the:el-complex}, where we now 
identify nodes $v\in V$ based on the their label $\lambda(v)$ instead of on a 
single type.

\begin{restatable}{lemma}{LemIntAbsSize}\label{lem:int-abs-size}
 Let $\IntAbs$ be a hypothesis abstraction. 
Then, $\IntAbs$ can be transformed into a hypothesis abstraction 
$\IntAbs'=\tup{V',\lambda',\Rf',F'}$ where, for $\ell=\lvert 
T_{\Kmc\cup\Phi}\rvert\cdot 2^{\lvert 
T_{\Kmc\cup\Phi}\rvert}$, i) $V'$ contains at most $\ell$ 
internal nodes, ii) every $v\in V'$ has at most $\ell$ successors in $\Rf'$, and
iii)~every path of outgoing nodes contains at most $\ell$ nodes.%
\comment{
Then, there exists a $\Sigma$-complete, $\ALC$-conform interpretation 
abstraction 
$\IntAbs'=\tup{V',\lambda',\Rf',F'}$ that explains $\Phi$ and for which

\begin{enumerate}
    \item $V'$ has at most $2^{\lvert\types{\Kmc\cup\Phi}\rvert}$ internal 
nodes,
    \item every node has at most $2^{\lvert\types{\Kmc\cup\Phi}\rvert}$ 
successors,
    \item every acyclic path that does not go through a named node has a length 
bounded by $\lvert\ind{\Kmc\cup\Phi}\rvert\times 
2^{\lvert\types{\Kmc\cup\Phi}\rvert}$, and
    \item if $\IntAbs$ explains $\Phi$, then so does $\IntAbs'$.
\end{enumerate}
}%
\end{restatable}%
\comment{
\begin{proof}[Proof sketch.]
 The crucial observation is that we can eliminate any repetitions of 
 labels that occur in a path, provided that the path does not go through an 
named node. Specifically, if such a path goes through nodes $v_1,v_2\in T$ 
s.t. $\lambda(v_1)=\lambda(v_2)$, we can remove the nodes between $v_1$ and 
$v_2$ and connect $v_1$ directly to the successors of $v_2$.
This way, we can bound the length of paths between individual 
nodes, as well as of outgoing paths, to double exponential length. Similarly, 
we can merge pairs of successors of the same node that have the same label.
\end{proof}
}%
The final ingredient to establish Theorem~\ref{the:upper-alc-complex} 
is the following lemma.
\begin{restatable}{lemma}{LemIntAbsToABox}\label{lem:int-abs-to-abox}
 For every $\ALC$-conform, $\Sigma$-complete interpretation abstraction 
$\IntAbs=\tup{V,\lambda,s,\Rf,F}$, there exists an ABox $\Hmc$ s.t. i) the 
models of $\Kmc\cup\Hmc$ are 
exactly the models of $\Kmc$ accepted by $\IntAbs$, ii)
$\lvert\ind{\Hmc}\rvert\leq\lvert V\rvert$, and iii) for every 
$a\in\ind{\Hmc}$, 
$\Hmc$ contains one assertion $C(a)$, with $\size{C}$ 
exponentially bounded by the path lengths between outgoing nodes in~$\IntAbs$.
\end{restatable}

Unfortunately, interpretation abstractions cannot be as easily constructed by 
a 
deterministic procedure as we did for Theorems~\ref{the:compute-flat} 
and~\ref{the:el-complex},
\comment{
The reason is that we have to non-deterministically 
select a type for each node in order to verify whether an interpretation 
abstraction satisfies any model of $\Kmc$ - this cannot be decided by a simple 
propagation strategy regarding the internal nodes, and consequently, it is not 
clear which role assignments between internal nodes we are allowed to put.
}
as there can to be non-trivial interactions between connected 
internal nodes, and we only have a double exponential upper bound on their 
number. To decide \ALC abduction, we can however guess an interpretation 
abstraction within 
the bounds of Lemma~\ref{lem:int-abs-size}, and then guess assignments of types 
to nodes to obtain its models. We thus obtain the following theorem.

\comment{
We were not yet able to obtain tight bounds for the decision problem, and 
currently believe that the problem could indeed be $\NTwoExpTime$-hard or more. 
Using Lemma~\ref{lem:int-abs-size}, we can however obtain the 
following upper bound.
}
\begin{restatable}{theorem}{TheDecideComplexALCAbduction}
$\ALC$ ABox abduction is in $\NTwoExpTime^\NP$. 
\end{restatable}
\comment{
\begin{proof}[Proof sketch]
 We guess an $\ALC$-conform, $\Sigma$-complete interpretation 
abstraction $\IntAbs$ within the size bounds of~\ref{lem:int-abs-size}. Because 
we can reuse outgoing nodes to which there is the same path from an individual 
node, it suffices to guess one such interpretation abstraction that is of at 
most double exponential size.
We then verify it 
abstracts some model of $\Kmc$ by guessing for each node $v$ a type 
$t\in\lambda(v)$ and checking whether those types can be implemented by a model 
of $\Kmc$. This check is possible in polynomial time by just considering a 
node, its assigne type, and its successors. 
Finally, we verify that it explains the observation by making 
another such guess to verify whether there exists a model $\Imc$ of $\Kmc$ s.t. 
$\Imc\not\models\Phi$. 
\end{proof}
}

%% file: minimal-solutions.tex
\section{Size-Restricted Abduction}
Because hypotheses can become very large, a natural requirement is to compute  
hypotheses of minimal or bounded size. 
We here obtain the following complexities.

\begin{theorem} 
 Size restricted $\Lmc$ ABox abduction is
\begin{itemize}
  \item $\NP$-complete for $\Lmc=\EL$,
  \item $\NExpTime$-complete for $\Lmc=\ELbot$,
  \item $\NExpTime^\NP$-complete for the flat variant and 
$\Lmc\in\{\ALC,\ALCI\}$, and
  \item in $\TwoExpTime$ for $\Lmc=\ALCIQ$.
 \end{itemize}
\end{theorem}
The upper bounds are based on guess-and-check algorithms. For \EL, we exploit 
the fact that, by Theorem~\ref{the:flat-size-upper}, we can always find a 
solution of polynomial size. For $\ELbot$, we note that the 
size of the hypothesis is exponentially bounded by the number of bits used for 
the size bound $k$. 
The $\NExpTime^\NP$-upper bound can be obtained by a refinement of the procedure 
used in the proof for Theorem~\ref{the:compute-flat}.
For the double exponential upper bound, we iterate over the double exponentially 
many possible KBs within the size and signature bounds---independent on whether 
we are interested in flat or complex solutions---and then check for entailment 
in time exponential in the size of the current solution. The lower bounds are 
provided by the following lemmas.
 
\begin{restatable}{lemma}{LemELMinimal}
  Size-restricted $\EL$ abduction is \NP-hard.
\end{restatable}
\begin{proof}[Proof sketch]
  We reduce the \NP-complete problem CNF-SAT to deciding whether a given 
  signature-based problem has a hypothesis of size at most $k$. Let 
$\phi=c_1\wedge\ldots\wedge c_n$ be a CNF formula over propositional variables 
$p_1,\ldots,p_m$. 
$\Kmc$ contains the following axioms:
  \begin{align*}
    \textsf{True}\sqsubseteq P &\quad
    \textsf{False}\sqsubseteq P \quad
    \exists r.\textsf{True}\sqsubseteq C \quad
    \exists s.\textsf{False}\sqsubseteq C \\
    r(c_i,p_j) &\text{ for every }i\in\iinv{1}{n}, j\in\iinv{1}{m}, \text{ if 
$p_j\in c_j$}\\
    s(c_i,p_j) &\text{ for every }i\in\iinv{1}{n}, j\in\iinv{1}{m}, \text{ if 
$\neg p_j \in c_j$}
  \end{align*}
  $\Phi$ contains $P(p_i)$ for every $i\in\iinv{1}{m}$ and $C(c_i)$ for every 
$i\in\iinv{1}{n}$. 
Finally, $\Sigma=\{\textsf{True},\textsf{False}\}$. %
$\tup{\Kmc,\Phi,\Sigma}$ has 
a hypothesis of size at most $2m$ iff $\phi$ is satisfiable. 
\end{proof}

\begin{restatable}{lemma}{LemELBotMinimal}\label{lem:elbot-minimal}
   Size restricted $\ELbot$ abduction is $\NExpTime$-hard.
\end{restatable}
\newcommand{\Start}{\ensuremath{\mathsf{Start}}\xspace}
\newcommand{\End}{\ensuremath{\mathsf{End}}\xspace}
\begin{proof}[Proof sketch]
  The hardness follows from a reduction of the $\NExpTime$-complete 
  \emph{exponential tiling problem}, which is given by a tuple $\tup{T,T_I,t_e, 
V,H,n}$ of a set $T$ of tile types, a sequence $T_I\in T^*$ of initial tiles, a 
final tile $t_e$, vertical and horizontal tiling conditions $V,H\subseteq 
T\times T$, and a number $n$ in unary encoding. A solution to this problem is 
then a \emph{tiling}, as a function 
$f:\iinv{1}{2^n}\times\iinv{1}{2^n}\rightarrow T$ assigning tiles to 
coordinates, s.t. the first tiles are as in $T_I$, $f(2^n,2^n)=t_e$, and 
that obeys the vertical and horizontal tiling conditions~\cite{TILINGS}.

  In the reduction, concept names \Start and \End 
respectively mark the initial and the final 
tile. We implement two binary counters $X$ and $Y$ as for 
Theorem~\ref{the:exp-lower-bound-flat} which are decremented over the roles $x$ 
and $y$, and encode 
the coordinates of the tiles. Each tile type $t\in T$ is represented by a 
concept name $A_t$. We enforce the horizontal tiling conditions using CIs 
\[                                                      
    \exists x.A_t\sqcap A_{t'}\sqsubseteq\bot
    \qquad \text{ for each }\tup{t,t'}\in (T\times T)\setminus H
\]
and correspondingly for the vertical 
conditions. The (hidden) concept name $B\not\in\Sigma$ is used to ensure that the hypothesis 
contains at least one individual per coordinate. This name is initialised by 
the 
individual satisfying \Start, and then propagated in $x$ and $y$ 
direction, provided that a tiling type is associated.
The observation to be explained is $\End(a)$, where $\End$ occurs in the following 
CI:
\[
  \bigsqcap_{i=1}^n X_i\sqcap
  \bigsqcap_{i=1}^n Y_i\sqcap
  B\sqcap
  A_{t_e}
   \sqsubseteq\End 
  \]
and the abducibles are
\[
 \Sigma=\{\Start,x,y\}\cup\{A_t\mid t\in T\}.
\]
Without the size restriction, a valid hypothesis corresponds to a binary tree 
with tile types associated to each node, and tiling conditions ensured along 
the $x$- and $y$-successors. To make sure it forms a $2^n\times 2^n$ grid, we 
choose the size $k$ appropriately in a way that every coordinate can be used at 
most once. Valid hypotheses of size $k$ then correspond to solutions to the 
tiling problem.
\end{proof}

To present the proof idea for the $\NExpTime^\NP$-hardness result more 
concisely, we introduce a new tiling problem.
 
\begin{definition}
 A \emph{$\NExpTime^\NP$-tiling problem} is given by a tuple 
$\tup{T,T_I,t_e,H_1,V_1,H_2,V_2,n}$, where $\tup{T,T_I,t_e,H_1,V_1,n}$ is an 
exponential tiling problem, $H_2,V_2\subseteq T\times T$ are additional 
tiling conditions, and for which we want to decide the existence of 
a valid tiling 
$f:\iinv{1}{2^n}\times\iinv{1}{2^n}\rightarrow T$ for the tiling problem 
$\tup{T,T_I,t_e,V_1,H_1,n}$, 
 s.t. for no $i\in\iinv{1}{2^n}$, there exists a valid tiling for the 
tiling problem $\tup{T,f(i), t_e,V_2,H_2, n}$, where $f(i)$ denotes 
the $i$th row of the tiling $f$.
\end{definition}
In other words, we have to find a tiling using conditions 
$H_1$ and $V_1$, while avoiding any rows that can be first row of any tiling for conditions $H_2$ 
and $V_2$. 
\begin{restatable}{lemma}{LemNewTiling}\label{t:new-tiling}
 The $\NExpTime^\NP$-tiling problem is $\NExpTime^\NP$-hard.
\end{restatable}

\begin{restatable}{lemma}{LemALCMinimal}
Size-restricted $\ALC$ abduction is $\NExpTime^\NP$-hard.
\end{restatable}
\begin{proof}[Proof sketch]
 We modify the construction for Lemma~\ref{lem:elbot-minimal} to encode the 
$\NExpTime^\NP$-tiling problem. We now use 3 roles $x$, $y$, $z$ and 
corresponding binary counters so that, together with the size restriction, each 
hypothesis will have the shape of a cube. The bottom side of this cube has to 
correspond to a tiling for $\tup{T,T_I,t_e,V_1,H_1,n}$, which can be achieved 
using similar axioms as for Lemma~\ref{lem:elbot-minimal}. For nodes outside of 
the bottom side of the cube, we require the use a different set of concept 
names for the tile types, which are of the form $A^*_t$, and for which we have 
the axiom $T^*\sqsubseteq\bigsqcup_{t\in T}A_t^*$. We use
\[
   \Sigma=\{\Start,x,y,z,T^*\}\cup\{A_t\mid t\in T\},
\] and again require every 
coordinate to be assigned some tile type. For the coordinates 
outside the bottom side, we have to use the concept name $T^*$ to assign tile 
types, which leaves the precise selection of the tile type to the different 
models of the hypothesis. We detect tiling errors in the different $x\times 
z$-squares with the following axioms %
\[
 \exists x.A_t^*\sqcap A_{t'}^*\sqsubseteq B_3
      \qquad\text{ for }\tup{t,t'}\in(T\times T)\setminus H_2 
\] and correspondingly for $V_2$.
This information is propagagated along the succeeding coordinates 
so that the observation $\End(a)$ 
is only entailed if every model of the hypothesis encodes a tiling error on 
each of the $x\times z$ squares.%
\end{proof}

%% file: conclusion.tex
\section{Outlook}

We believe that our results for complex abduction in $\ALC$ can be extended 
to $\ALCI$, and that  %
the bound for 
size restricted \ALCQI abduction in 
is tight. A question 
is whether we can improve the $\NTwoExpTime^\NP$-bound for the most general 
variant of our abduction problem. Apart from that, we want to 
investigate our setting for observations formulated as conjunctive queries, 
which would allow us to explain negative query answers~\cite{DL_LITE_ABDUCTION}.
Another interesting question is what happens if we allow fresh individual names 
for abduction with ontologies formulated using existential rules. For the 
$\ELbot$-variant, we are currently working on a practical method for computing 
size-minimal flat hypotheses.

%% file: appendix-flat-solutions.tex
\section{Flat ABox Hypotheses}

\newif\ifselfcontained\selfcontainedfalse

\TheFlatSizeUpper*
\begin{proof}

\ifselfcontained
 Let $\Hmc_0$ be a hypothesis for $\tup{\Kmc,\Phi,\Sigma}$. Based on $\Hmc_0$, 
we construct a 
hypothesis within the required size bounds. We first consider the case where $\Lmc=\ALCI$, and then 
the case where $\Lmc=\ELbot$.

\medskip
 
 First assume $\Lmc=\ALCI$.
 Let $\Imc$ be a model of $\Kmc\cup\Hmc_0$. For every individual name 
 $a\in\ind{\Hmc_0}$, we define its \emph{type in $\Imc$} as 
 \[
   \type{a}=\{C\in\sub{\Kmc}\cup\sub{\Phi}\mid\Imc\models C(a)\}.   
 \]
For every such type, we introduce a fresh individual name $c_\type{a}$. 
The 
number of such individual names is exponentially bounded by the size of~$\Kmc$ 
and~$\Phi$. 
We define 
a function $h:\ind{\Hmc_0}\rightarrow\NI$ by setting $h(a)=a$ if 
$a\in\ind{\Kmc\cup\Phi}$, and 
otherwise $h(a)=c_\type{a}$.
Based on $h$ and $\Hmc_0$, we construct the size bounded hypothesis $\Hmc$ as follows:
\begin{align*}
       &\{A(h(a))\mid A(a)\in\Hmc_0\} \\
  \cup &\{r(h(a),h(b))\mid r(a,b)\in\Hmc_0\} 
\end{align*}
Since $\Hmc$ is flat and contains only exponentially many individual names, 
$\Hmc$ is exponential in the size of $\Kmc$ and $\Phi$. Note also that $h$ is a 
homomorphism from 
$\Hmc_0$ into $\Hmc$ that maps each individual name from $\ind{\Kmc\cup\Phi}$ 
to 
itself, and may map several individual names from $\Hmc_0$ onto the same individual name in $\Hmc$. 

We argue that $\Hmc$ is a hypothesis for $\tup{\Kmc,\Phi,\Sigma}$. By 
construction, $\sig{\Hmc}\subseteq\sig{\Hmc_0}$, and thus $\sig{\Hmc}\subseteq\Sigma$. 
\else %
Let $\Hmc_0$ be a hypothesis for $\tup{\Kmc,\Phi,\Sigma}$. 
We first complete the argument in the main text for $\Lmc=\ALCI$, and then for 
$\Lmc=\EL$. 

\medskip

For $\Lmc=\ALCI$, we construct $\Hmc$ and $h$ as in the main text based on an 
arbitrary hypothesis 
$\Hmc_0$ and an arbitrary model $\Imc$ of $\Kmc\cup\Hmc_0$. Note that there can be at most 
exponentially many distinct types occuring in $\Imc$ (one for each subset of 
\mbox{$\sub{\Kmc\cup\Phi}$}). Consequently, $\Hmc$ contains at most exponentially many individual 
names, and, because it is flat, at most exponentially many assertions.
It remains to show that $\Hmc$ 
satisfies~\ref{itm:consistency} and~\ref{itm:entailment}. 
\fi %

For \ref{itm:entailment}, we need to show that $\Kmc\cup\Hmc\not\models\bot$. For this, we 
construct the interpretation $\Jmc$ based on the types occurring in $\Imc$. 
The domain of $\Jmc$ contains an element for every individual in $\Kmc\cup\Phi$ 
and additionally 
types occurring in $\Imc$. We first define a superset of the domain:
\[
 \Delta=\ind{\Kmc\cup\Phi}\cup\{\typeI{d}{\Imc}\mid d\in\Delta^\Imc\}
\]
We define a function $g:\Delta^\Imc\rightarrow\Delta$ by setting:
\[
 g(d)=\begin{cases}
       a & \text{ if $d=a^\Imc$ and $a\in\ind{\Kmc\cup\Phi}$} \\
       \typeI{d}{\Imc} & \text{otherwise}
      \end{cases}
\]

The interpretation $\Jmc$ is now defined as follows for all $a\in\NI$, 
$A\in\NC$ and $r\in\NR$:
\begin{itemize}
 \item $\Delta^\Jmc=\{g(d)\mid d\in\Delta^\Imc\}$
 \item $a^\Jmc=\begin{cases}
                a & \text{if $a\in\ind{\Kmc\cup\Phi}$} \\
                t & \text{if $a\in\ind{\Hmc}\setminus\ind{\Kmc\cup\Phi}$} \\
                  & \text{\ \ and $a=b_t$ for $t\in T_{\Kmc\cup\Phi}$}\\
                \text{arbitrary} & \text{otherwise}
               \end{cases}$
 \item $A^\Jmc=\{g(d)\mid d\in A^\Imc\}$,
 \item $r^\Jmc=\{\tup{g(d),g(e)}\mid \tup{d,e}\in r^\Imc\}$
\end{itemize}

It can be shown by structural induction that for every concept 
$C\in\sub{\Kmc\cup\Phi}$ and 
$d\in\Delta^\Imc$, $d\in C^\Imc$ iff $g(d)\in C^\Jmc$, which implies that $\Jmc$ is a model of 
$\Kmc$. Note however that this would not work for $\Lmc=\ALCF$, since the transformation can 
change the number of role-successors. 

We can also show that $\Jmc\models\Hmc$.
First observe that for every 
$a\in\ind{\Kmc\cup\Hmc_0}$ and $b\in\ind{\Kmc\cup\Hmc}$, $h(a)=b$ iff $g(a^\Imc)=b^\Jmc$: if 
$a\in\ind{\Kmc\cup\Phi}$, we have by definition $h(a)=a$, $g(a^\Imc)=a$ and 
$a=a^\Jmc$. Otherwise, 
we have $h(a)=b_{\typeI{a^\Imc}{\Imc}}$ , $g(a^\Imc)=\typeI{a^\Imc}{\Imc}$ and 
$b_{\typeI{a^\Imc}{\Imc}}^\Jmc=\typeI{a^\Imc}{\Imc}$. In both 
cases, we have $h(a)=b$ iff $g(a^\Imc)=b^\Jmc$. 

Assume $A(b)\in\Hmc$. Then $b=h(a)$ for some 
$a\in\ind{\Kmc\cup\Hmc_0}$ and $A(a)\in\Hmc_0$. We have $g(a^\Imc)=b^\Jmc$ by the above 
observation, and we have already pointed out that this implies for every 
$C\in\sub{\Kmc\cup\Phi}$ 
that $a^\Imc\in C^\Imc$ iff $b^\Jmc\in C^\Jmc$. Since $\Imc\models\Hmc_0$, we have $a^\Imc\in 
A^\Imc$, and thus $b^\Jmc\in A^\Jmc$, so that $\Jmc\models A(b)$.

Assume $r(b_1,b_2)\in\Hmc$. Then, there exists $r(a_1,a_2)\in\Hmc_0$ s.t. $h(a_1)=b_1$ and 
$h(a_2)=b_2$. Again we can conclude that $g(a_1^\Imc)=b_1^\Jmc$, $g(a_2^\Imc)=b_2^\Jmc$ and 
$\tup{a_1^\Imc,a_2^\Imc}\in r^\Imc$, which 
by definition of $r^\Jmc$ implies that $\tup{b_1^\Jmc,b_2^\Jmc}\in r^\Jmc$, so that $\Jmc\models 
r(b_1,b_2)$. 

We have shown that $\Jmc\models\Kmc\cup\Hmc$, and thus that $\Kmc\cup\Hmc$ is 
satisfiable, that is, that $\Hmc$ satisfies~\ref{itm:consistency}. 

It remains to show \ref{itm:entailment}, that is,  
$\Kmc\cup\Hmc\models\Phi$. Assume $\Kmc\cup\Hmc\not\models\Phi$, that is, there 
exists some 
assertion $\alpha\in\Phi$ s.t. $\Kmc\cup\Hmc\not\models\alpha$. If $\alpha$ is 
of the form 
$r(a,b)$, then this axiom must explicitly occur in either $\Kmc$ or $\Hmc_0$, 
since $\ALCI$ is not 
expressive enough to enforce such entailments in a different way). If $r(a,b)\in\Kmc$, trivially 
$\Kmc\cup\Hmc\models r(a,b)$. If $r(a,b)\in\Hmc_0$, by construction also $r(a,b)\in\Hmc$. 

Assume $\alpha$ is of the form 
$C(a)$, and there exists a model $\Jmc$ of $\Kmc\cup\Hmc$ s.t. $\Jmc\not\models C(a)$. Using the 
homomorphism $h:\ind{\Hmc_0}\mapsto\ind{\Hmc}$, we can transform $\Jmc$ into a 
model 
$\Jmc'$ of $\Kmc\cup\Hmc_0$ s.t. $\Jmc'\not\models 
C(a)$. Specifically, we define $\Jmc'$ by
\begin{itemize}
 \item $\Delta^{\Jmc'}=\Delta^\Jmc\cup\{a^{\Jmc'}\mid a\in\ind{\Hmc_0}\}$,
 \item $A^{\Jmc'}=A^{\Jmc}\cup\{a^{\Jmc'}\mid a\in\ind{\Hmc_0}$, $h(a)^\Jmc\in 
A\}$ for all $A\in\NC$, and
 \item $r^{\Jmc'}=r^{\Jmc}\cup\{\tup{a^{\Jmc'}$, $b^{\Jmc'}}\mid 
a,b\in\ind{\Hmc_0}$, $\tup{h(a)^\Jmc,h(b)^\Jmc}\in r^{\Jmc'}\}$ for all 
$r\in\NR$.
\end{itemize}
$\Jmc'$ can be seen as the result of duplicating some elements from $\Jmc$ 
based on $h$, 
while keeping their concept names and role successors. Thus, it is easy to see that the 
transformation is type-preserving, that is, for every $a\in\ind{\Hmc}$ and $b\in\ind{\Hmc_0}$ s.t. 
$h(b)=a$, we have $\typeI{a^\Jmc}{\Jmc}=\typeI{b^{\Jmc'}}{\Jmc'}$. As a consequence, 
$\Jmc'\models\Kmc$ and 
$\Jmc'\not\models C(a)$. Furthermore, by construction, $\Jmc'\models\Hmc_0$. We obtain that 
$\Kmc\cup\Hmc_0\not\models\Phi$, and thus contradicting that $\Hmc_0$ was a 
hypothesis 
for $\Phi$ to begin with. Consequently, there cannot be such a model $\Jmc$ of 
$\Kmc\cup\Hmc$, and 
we must have $\Kmc\cup\Hmc\models\Phi$. 

We have shown that $\Hmc$ satisfies Conditions~\ref{itm:consistency}--\ref{itm:entailment}, and 
thus that it is an exponentially sized hypothesis for $\tup{\Kmc,\Phi,\Sigma}$. 

\medskip

Now assume $\Lmc=\EL$. We then construct a polynomially sized hypothesis by 
\begin{align*}
 \Hmc= &\{A(a)\mid A\in\Sigma\cap\NC,a\in\ind{\Kmc\cup\Phi}\} \\
       &\cup\{r(a,b)\mid r\in\Sigma\cap\NR,a,b\in\ind{\Kmc\cup\Phi}\}
\end{align*}
Again we can define a homomorphism $h$ from $\Hmc_0$ onto $\Hmc$: this time, we 
map every 
individual name not in 
$\ind{\Kmc\cup\Phi}$ to some randomly selected individual name. We have 
$\sig{\Hmc}\subseteq\Sigma$ by 
construction, and $\Kmc\cup\Hmc\not\models\bot$ since every $\EL$ KB is 
consistent. 
$\Kmc\cup\Hmc\models\Phi$ can be shown as for $\Lmc=\ALCI$ using the 
homomorphism $h$. It follows 
that $\Hmc$ is a polynomially sized hypothesis for $\tup{\Kmc,\Phi,\Sigma}$.

Finally, for $\Lmc=\ALCF$, the existence of any upper bound would allow us to 
decide existence of abductive solutions by a simple guess-and-check algorithm. 
In 
particular, this would 
allow to decide the IQ-query emptiness problem which asks, given a TBox $\Tmc$, 
a signature 
$\Sigma$ and a concept $C$, whether there exist any flat ABoxes in $\Sigma$ that 
entail $C(a)$ for 
some individual name $a\in\NI$. As 
shown in~\cite{QUERY_EMPTINESS}, this problem is undecidable for $\ALCF$ 
TBoxes, 
however, it can be easily reduced to flat $\ALCF$ abduction.
\end{proof}

\TheFlatCompute*
\begin{proof}
 \newcommand{\Inv}[1]{\textsf{Inv}(#1)}
We only need to consider the upper bounds for the cases $\Lmc=\ELbot$ and 
$\Lmc=\ALCI$. 

We first describe how to compute a maximal set $T_{\Kmc\cup\Phi}$ of types for 
$\Kmc\cup\Phi$ using type 
elimination. This construction is fairly standard and only included here for 
self-containment. 
For a role $R$, we set $\Inv{r}=r^-$ 
and $\Inv{r^-}=r$. 

We start with the 
set $T_0=\{t\mid t\subseteq\sub{\Kmc}\cup\sub{\Phi}\}$ from which we construct 
a 
sequence 
$T_0$, $T_1$, $\ldots$, $T_n$ of such sets. 
For a given such set $T_i$, a type $t\subseteq\sub{\Kmc\cup\Phi}$
and a role $R$, we denote by $\sucCand{T_i}{t}{R}$ the set of all types 
$t'\in T_i$ s.t. for every $\exists R.C\in (\sub{\Kmc\cup\Phi}\setminus t)$, 
$C\not\in t'$, and for every $\exists\Inv{R}.C\in(\sub{\Kmc\cup\Phi}\setminus 
t')$, $C\not\in t$. $\sucCand{T_i}{t}{R}$ contains the \emph{$R$-successor 
candidates} of $t$ within $T_i$, that is, types from $T_i$ that an 
$R$-successor of an instance of $t$ could have.
For each $i\in\iinv{1}{n}$, $T_i$ is obtained from $T_{i-1}$ 
by 
removing an element $t$ that fails to satisfy any of the following conditions: 
\begin{enumerate}
 \item $\bot\not\in t$,
 \item for every $C\sqcap D\in\sub{\Kmc\cup\Phi}$,  $C\sqcap D\in t$ iff 
$\{C,D\}\subseteq t$,
 \item for every $\neg C\in\sub{\Kmc\cup\Phi}$, $C\in t$ iff $\neg C\not\in t$,
 \item if $C\in t$ and $C\sqsubseteq D\in\Kmc$, then $D\in t$, and
 \item\label{itm:successor} if $\exists R.C\in t$, then there exists 
$t'\in\sucCand{T_{i-1}}{t}{R}$ 
s.t. \mbox{$C\in t'$}.
\end{enumerate}

 We set $T_{\Kmc\cup\Phi}=T_n$.
 Since $T_0$ contains exponentially many elements, 
$T_{\Kmc\cup\Phi<}$ is obtained after at most exponentially 
many steps, each taking polynomial time. 

We define a \emph{selector function} to be any function 
\mbox{$s:\ind{\Kmc\cup\Phi}\rightarrow T_{\Kmc\cup\Phi}$}. We note that there 
are at most 
exponentially many such 
selector functions, specifically, 
\begin{align*}
    \lvert T\rvert^{\lvert\ind{\Kmc\cup\Phi}\rvert} &\leq
\left(2^{\lvert\sub{\Kmc\cup\Phi}\rvert}\right)^{\lvert\ind{\Kmc\cup\Phi}\rvert}
\\
    &\leq 2^{\size{\Kmc\cup\Phi}^2}
\end{align*}
We call such a selector function \emph{compatible with $\Kmc$} if for every 
$A(a)\in\Kmc$, $A\in 
s(a)$, and for every $r(a,b)\in\Kmc$, 
$s(b)\in\sucCand{T_{\Kmc\cup\Phi}}{s(a)}{r}$.

For every selector function $s$ compatible with $\Kmc$, we build a hypothesis 
candidate $\Hmc_s$ as 
follows. To every type $t\in T_{\Kmc\cup\Phi}$, we assign an fresh individual 
name $b_t$. The set
\[
    I=\ind{\Kmc\cup\Phi}\cup\{b_t\mid t\in T_{\Kmc\cup\Phi}\}
\]
contains the individual names in 
our hypothesis, to 
which we extend $s$ by setting $s(b_t)=t$. $\Hmc_s$ is then defined as 
\begin{align*}
 \Hmc_s =& \{A(a)\mid a\in I, A\in s(a)\cap\Sigma\}\\
        &\cup \{r(a,b)\mid a,b\in I, r\in\Sigma, 
s(b)\in\sucCand{T_{\Kmc\cup\Phi}}{s(a)}{r}\}
\end{align*}

\textbf{Claim 1.} \emph{If $s$ is a selector function compatible with 
$\Kmc$, then $\Kmc\cup\Hmc_s\not\models\bot$.}

\emph{Proof of claim.} We construct a model $\Imc$ for $\Kmc\cup\Hmc_s$ based on 
$T_{\Kmc\cup\Phi}$ and $s$: 
\begin{itemize}
 \item $\Delta^\Imc=I$,
 \item for all $a\in I$: $a^\Imc=a$
 \item $A^\Imc=\{a\in I\mid A\in s(a)\cap\NC\}$,
 \item $r^\Imc=\{\tup{a,b}\in I\times I\mid 
s(b)\in\sucCand{T_{\Kmc\cup\Phi}}{s(a)}{r}\}$
\end{itemize}
It is now standard to show by induction on the structure of concepts that for 
every $a\in\NI$ and 
$C\in s(a)$, $a\in C^\Imc$. From the construction of the type set 
$T_{\Kmc\cup\Phi}$ it then follows that $\Imc$ 
satisfies all GCIs in $\Kmc$. Since $s$ is compatible with $\Kmc$, it follows 
that 
$\Imc$ also satisfies all assertions in $\Kmc$. Finally, by construction 
$\Imc\models\Hmc_s$.
\hfill$\blacksquare$

\textbf{Claim 2.} \emph{If there exists a flat hypothesis for 
$\tup{\Kmc,\Phi,\Sigma}$, then there 
exists a selector function $s$ s.t. $\Hmc_s$ is such a hypothesis.}

\emph{Proof of claim.} Assume there exists a flat hypothesis $\Hmc_0$ for 
$\tup{\Kmc,\Phi,\Sigma}$. 
Since for every selector function $s$ compatible to $\Kmc$, $\Hmc_s$ 
satisfies~\ref{itm:consistency} by Claim 1 and~\ref{itm:signature} by 
construction, we only need to show that there exist some such $s$ s.t. $\Hmc_s$ 
satisfies~\ref{itm:entailment}. We use the construction from the proof for 
Theorem~\ref{the:flat-size-upper} to obtain, based on $\Hmc_0$, an 
assignment of types $\type{a}$ to every $a\in\ind{\Kmc\cup\Phi}$, as well as 
the 
exponentially bounded hypothesis $\Hmc$. Define $s$ s.t. 
$s(a)=\typeI{a^\Imc}{\Imc}$ for every $a\in\ind{\Kmc\cup\Phi}$. $s$ is 
compatible with 
$\Kmc$, and $\Hmc\subseteq\Hmc_s$. Consequently, 
$\Kmc\cup\Hmc_s\models\Phi$ and $\Hmc_s$ 
satisfies~\ref{itm:consistency}--\ref{itm:entailment}.\hfill$\blacksquare$

Claim~2 allows to obtain an \ExpTime-decision procedure for the case where 
$\Lmc=\ELbot$: 
specifically, we first compute $T_{\Kmc\cup\Phi}$ in exponential time, and then 
iterate over the exponentially 
many possible selector functions $s$ and check whether $\Hmc_s$ 
satisfies~\ref{itm:consistency} 
and~\ref{itm:entailment}. Since entailment checking in $\ELbot$ can be decided 
in polynomial time 
and each $\Hmc_s$ is exponentially large, the entire procedure runs in 
deterministic exponential 
time.

For $\Lmc=\ALCI$, entailment checking is $\ExpTime$-hard, so that this procedure 
would only yield 
in $\TwoExpTime$-upper bound. To obtain a $\coNExpTime$ upper bound, we devise a 
non-deterministic 
procedure to determine whether there is \emph{no} flat hypothesis for 
$\tup{\Kmc,\Phi,\Sigma}$. For 
this, by Claim~1 and~2, we need to show that for every selector function $s$, 
there exists some 
model $\Imc$ of $\Kmc\cup\Hmc_s$ s.t. $\Imc\not\models\Phi$. For this, we first 
check whether 
for some $r(a,b)\in\Phi$, $r(a,b)\not\in\Hmc_s\cup\Kmc$. If this test fails, we 
guess a function 
$s_e:\ind{\Hmc_s}\rightarrow T_{\Kmc\cup\Phi}$ and check whether 
\begin{enumerate}[label=\textbf{s\arabic*}]
    \item\label{itm:compatible} $s_e$ is compatible with $\Kmc\cup\Hmc_s$ in the 
sense above, and
    \item\label{itm:concept} for some $C(a)\in\Phi$, $C\not\in s_e(a)$. 
\end{enumerate}
    If for every selector $s$ function, one of these checks is successful, we 
accept, and otherwise 
we reject. 
Since this non-deterministic procedure requires exponentially many steps units, 
we obtain the 
required $\coNExpTime$ upper bound, provided that the algorithm is correct.

\textbf{Claim 3.} \emph{$\Kmc\cup\Hmc_s\models\Phi$ iff for every 
$r(a,b)\in\Phi$, 
$r(a,b)\in\Hmc_s$, and no function $s_e:\ind{\Hmc_s}\rightarrow 
T_{\Kmc\cup\Phi}$ satisfies 
Conditions~\ref{itm:compatible} and~\ref{itm:concept}.}

\emph{Proof of claim.}
Assume $\Kmc\cup\Hmc_s\not\models\Phi$ and for every $r(a,b)\in\Phi$, 
$r(a,b)\in\Hmc_s$. There then 
exists a model $\Imc$ of $\Kmc\cup\Hmc_s$ s.t. $\Imc\not\models\Phi$. Define 
$s_e$ by setting
$s_e(a)=\typeI{a^\Imc}{\Imc}$, and we obtain from the fact that $\Imc$ is a model of 
$\Kmc\cup\Hmc_s$ that $s_e$ satisfies Condition~\ref{itm:compatible}, and from the fact that 
$\Imc\not\models\Phi$, that there is some $C(a)\in\Phi$ s.t. $a^\Imc\not\in 
C^\Imc$, which implies 
$C\not\in s_e(a)$ and thus Condition~\ref{itm:concept}.

For the other direction, first observe that if $r(a,b)\not\in\Hmc_s\cup\Kmc$ for some 
$r(a,b)\in\Phi$, we can always find a model $\Imc$ of $\Kmc\cup\Hmc_s$ s.t. 
$\Imc\not\models\Phi$. 
Assume for every $r(a,b)\in\Phi$, we have also $r(a,b)\in\Hmc_s\cup\Kmc$, and 
there exists a
function $s_e:\ind{\Hmc_s}\rightarrow T_{\Kmc\cup\Phi}$ that satisfies both 
Condition~\ref{itm:compatible} 
and~\ref{itm:concept}. We define 
\[
    s_e^*:\ind{\Hmc_s}\cup T_{\Kmc\cup\Phi}\rightarrow T_{\Kmc\cup\Phi}
\]
by setting 
$s_e^*(a)=s_a(a)$ if $a\in\ind{\Hmc_s}$ and $s_e^*(t)=t$ if $t\in 
T_{\Kmc\cup\Phi}$.
We define an interpretation $\Imc$ as follows:
\begin{itemize}
 \item $\Delta^\Imc=\{a^\Imc\mid a\in\ind{\Hmc_s}\}\cup T_{\Kmc\cup\Phi}$,
 \item for every $A\in\NC$, $A^\Imc=\{d\in\Delta^\Imc\mid A\in s_e^*(d)\}$
 \item for every $r\in\NR$, $r^\Imc=\{\tup{d,e}\in\Delta^\Imc\times\Delta^\Imc\mid 
 s_e^*(e) \in \sucCand{T_{\Kmc\cup\Phi}}{s_e^*(d)}{r}\}$
\end{itemize}
One easily verifies that for every $d\in\Delta^\Imc$, $\typeI{d}{\Imc}=s_e^*(t)$. This implies that 
$\Imc$ satisfies every GCI in $\Kmc$, and by Condition~\ref{itm:concept} that $\Imc\not\models 
C(a)$ for some $C(a)\in\Phi$. Using Condition~\ref{itm:compatible}, one also 
verifies that $\Imc$ 
is a model of every assertion in $\Kmc\cup\Hmc_s$. 
\phantom{,,} \hfill$\blacksquare$

We obtain that our non-deterministic procedure is correct, and correspondingly that for 
$\Lmc=\ALCI$, existence of flat hypotheses is in \coNExpTime.
\end{proof}

%% file: appendix-complex-solutions-EL.tex
\section{Complex ABox Hypotheses for $\ELbot$}

We show Theorem~\ref{the:el-complex} using three separate lemma for the lower 
bound of the size, the upper bound of the size, as well as for the decision 
procedure.

\begin{figure*}
 \begin{framed}
 \begin{align}
 \label{eq:elbot2-init}
  B & \sqsubseteq \overline{X}_1\sqcap\ldots\sqcap\overline{X}_n \\
  \label{eq:elbot2-counting-starts}
  \exists r.(\overline{X}_i\sqcap X_{i-1}\sqcap\ldots X_1) \sqcap
  \exists s.(\overline{X}_i\sqcap X_{i-1}\sqcap\ldots X_1) 
  &\sqsubseteq X_i 
  &&\text{ for }i\in\iinv{1}{n}\\
  \exists r.(X_i\sqcap X_{i-1}\sqcap\ldots\sqcap X_1) \sqcap
  \exists s.(X_i\sqcap X_{i-1}\sqcap\ldots\sqcap X_1) 
  &\sqsubseteq \overline{X}_i \
  &&\text{ for }i\in\iinv{1}{n}\\
  \exists r.(\overline{X_i}\sqcap\overline{X}_j) \sqcap
  \exists s.(\overline{X_i}\sqcap\overline{X}_j) 
  &\sqsubseteq \overline{X}_i 
  && \text{ for }i,j\in\iinv{1}{n}, j<i \\
  \exists r.(X_i\sqcap\overline{X}_j) \sqcap
  \exists s.(X_i\sqcap\overline{X}_j) 
  &\sqsubseteq X_i 
  && \text{ for }i,j\in\iinv{1}{n}, j<i 
  \\
  X_i\sqcap \overline{X_i}&\sqsubseteq\bot &&\text{ for }i\in\iinv{1}{n}
  \label{eq:elbot2-counting-ends}
  \\
  X_1\sqcap\ldots\sqcap X_n &\sqsubseteq A 
 \end{align}
 \end{framed}
 \caption{Background KB in $\ELbot$ enforcing a double exponentially large 
hypothesis if fresh individual names are not admitted.}
 \label{fig:lower-complex-elbot}
\end{figure*}

\begin{lemma}
 There exists a family of $\ELbot$ TBoxes $(\Tmc_n)_{n\geq 1}$ s.t. each $\Tmc_n$ is of size 
polynomial in $n$ and every hypothesis for $\tup{\Tmc_n,A(a),\{B,r,s\}}$ that uses only 
$a$ as individual name is of double exponential size.
\end{lemma}
\begin{proof}
$\Tmc_n$ can be seen in Figure~\ref{fig:lower-complex-elbot} and works as 
described in the main 
text. 
\end{proof}

\begin{lemma}\label{lem:complex-el-size-upper}
 Let $\tup{\Kmc,\Phi,\Sigma}$ be a signature based abduction problem s.t. $\Kmc$ 
and $\Phi$ are 
formulated in $\ELbot$. Then, there exists a solution to the abduction problem without fresh 
individual names iff there exists such a solution of double exponentially bounded size.
\end{lemma}
\begin{proof}
\newcommand{\rt}{\text{root}}
 Let $\Hmc_0$ be a hypothesis for $\tup{\Kmc,\Phi,\Sigma}$. We first create a 
flattened version of 
$\Hmc_0$, where we use a mapping \mbox{$\rt:\NI\nrightarrow\NI$} to keep track 
of a 
``root individual'' for each individual added. Initially, we set $\rt(a):=a$ 
for 
every $a\in\ind{\Hmc_0}$.
We then exhaustively apply the following rules:
\begin{enumerate}
 \item $(C\sqcap D)(a) \Longrightarrow C(a), D(a)$
 \item $(\exists r.C)(a_1) \Longrightarrow r(a_1,a_2), C(a_2)$, where 
 $a_2=b_{\rt(a_1),C}$ is fresh 
\end{enumerate}
Denote the resulting ABox by $\Hmc_1$ and note that $\Hmc_1$ is a flat hypothesis for 
$\tup{\Kmc,\Phi,\Sigma}$. Moreover, $\Hmc_1$ is \emph{forest-like} in the 
sense that the fresh
role-descendants of any fixed individual name~$a$ from $\Hmc_0$ form a tree. 
Specifically, for 
every fresh individual name $b_{a,C}$ there exists exactly one role assertion of the form 
$r(a',b_{a,C})$. 

Similar to the proof for Theorem~\ref{the:flat-size-upper}, we take a 
random model $\Imc$ of $\Kmc\cup\Hmc_1$, which allows us to assign a type 
$\typeI{a^\Imc}{\Imc}$ to every individual in $\Hmc_1$. 

We use this type to reduce the path length from any individual from 
$\Hmc_0$ to an individual in $\Hmc_1$ by removing repeated types on each path. To keep track of our 
changes, we use a mapping $h$ which is initialised with $h(a)=a$ for every $a\in\ind{\Hmc_1}$ and 
updated during the construction. We then repeat the following steps exhaustively:
\begin{enumerate}
 \item If for two individual names $a$, $b$, $\rt(a)=\rt(b)$, 
$\typeI{a^\Imc}{\Imc}=\typeI{b}{\Imc}$ and there is a 
non-empty path from $a$ to $b$, then replace the role assertion of the form 
$r(a',a)$ by 
$r(a',b)$, and remove all role assertions involving $a$. Set \mbox{$h(a):=b$}.
 \item If for any two role assertions of the form $r(a,b)$, $r(a,c)$, we have 
 $\rt(b)=\rt(c)$ and
$\typeI{b^\Imc}{\Imc}=\typeI{c^\Imc}{\Imc}$, then remove all role assertions 
involving $c$. Set 
\mbox{$h(c)::=b$}.
\end{enumerate}
Denote the resulting ABox by $\Hmc_2$. Different to the construction for 
Theorem~\ref{the:flat-size-upper}, $h$ is not a homomorphism.
We note that the operation preserves the forest shape, and moreover makes 
sure that every path from 
an individual name $a\in\ind{\Hmc_0}$ to an individual name $b_{a,C}$ is exponentially bounded, and 
for every individual name $a$, there are at most exponentially many role assertions of the form 
$r(a,b)$. As a result, for every individual name $a\in\ind{\Hmc_0}$, there are $(2^n)^{2^n}$ 
individual names $b$ with $\rt(a)=b$. We can thus build a double exponentially 
bounded hypothesis 
candidate by ``rolling up'' the ABox again, or exhaustively applying the following operation:
\begin{itemize}
 \item If $a$ has no role successor and $\rt(a)\neq a$, replace the role 
assertion $r(b,a)$ by 
 \[
  \exists r.\bigsqcap\{C\mid C(a) \text{ occurs in the current ABox }\}.
 \]
\end{itemize}
Denote the resulting hypothesis by $\Hmc$. We show that it satisfies 
Conditions~\ref{itm:consistency}--\ref{itm:entailment}. For this, it suffices to show that $\Hmc_2$ 
satisfies those conditions. Note that we already observed that $\Hmc_1$ is 
a flat hypothesis 
for $\tup{\Kmc,\Phi,\Sigma}$. 

Condition~\ref{itm:signature} that $\sig{\Hmc_2}\subseteq\Sigma$ is 
satisfied by construction since we do not introduce any new names. 
For~\ref{itm:consistency}, we need to find a model $\Jmc$ of $\Kmc\cup\Hmc_2$, 
which we do based our 
model~$\Imc$ of $\Kmc\cup\Hmc_1$. $\Jmc$ is constructed as follows:
\begin{itemize}
 \item $\Delta^\Jmc=\Delta^\Imc$,
 \item $a^\Jmc=a^\Imc$ for all $a\in\NI$,
 \item $A^\Jmc=A^\Imc$ for all $A\in\NC$,
 \item $r^\Jmc=\{\tup{a^\Jmc,h(b)^\Jmc}\mid \tup{a^\Imc,b^\Imc}\in r^\Imc\}\cup
 \{\tup{d,e}\mid e\neq a^\Imc$ for all $a\in\ind{\Hmc_0}$.
\end{itemize}
It can be shown by induction on the structure of concepts that for every concept 
$C\in\sub{\Kmc\cup\Phi}$ and $d\in\Delta^\Imc$, $d\in C^\Imc$ iff $d\in C^\Jmc$, 
which implies 
$\typeI{d}{\Imc}=\typeI{e}{\Jmc}$ and that $\Jmc$ is a model of 
$\Kmc\cup\Hmc_2$. 

For Condition~\ref{itm:entailment}, assume there is a model $\Imc_2$ of 
$\Kmc\cup\Hmc_2$ 
s.t. $\Imc_2\not\models\Phi$. Based on $\Imc_2$, we construct a model $\Imc_1$ 
of 
$\Kmc\cup\Hmc_1$.
Define $m:\Delta^{\Imc_2}\cup\ind{\Kmc\cup\Hmc_1}\rightarrow\Delta^{\Imc_2}$ by 
\[
    m(d)=\begin{cases}
          d             & \text{ if }d\in\Delta^{\Imc_2} \\
          h(d)^{\Imc_2} & \text{ if }d\in\ind{\Kmc\cup\Hmc} .
         \end{cases}
\]
Using $m$, we define $\Imc_1$ as follows
\begin{itemize}
 \item $\Delta^{\Imc_1}=\Delta^{\Imc_2}\cup\ind{\Kmc\cup\Hmc}$,
 \item $a^\Imc=a$ for all $a\in\ind{\Kmc\cup\Hmc}$,
 \item $A^{\Imc_1}=\{d\in\Delta^{\Imc_1}\mid m(d)\in A^{\Imc_2}\}$ for all $A\in\NC$, and
 \item $r^{\Imc_1}=\{\tup{d,e}\mid \tup{m(d),m(e)}\in r^{\Imc_2}\}$ for all $r\in\NR$.
\end{itemize}
Again it is straightforward to prove that the construction is type preserving, a model of 
$\Kmc\cup\Hmc_1$, and $\Imc_1\not\models\Phi$ for the same reason as for 
$\Imc_1$. It follows that 
there cannot be a model $\Imc_2$ of $\Kmc\cup\Hmc_2$ s.t. 
$\Imc_2\not\models\Phi$, and thus that 
$\Kmc\cup\Hmc_2$ satisfies Condition~\ref{itm:entailment}.
\end{proof}

\begin{lemma}\label{lem:elbot-decide-complex}
 It can be decided in exponential time whether a given $\ELbot$ abduction problem has a solution 
without fresh individuals.
\end{lemma}
\begin{proof}
 We use a similar procedure as in Theorem~\ref{the:compute-flat} using the in 
 exponential time computed set $T_{\Kmc\cup\Phi}$ of types for $\Kmc\cup\Phi$ 
and the guessed selector function $s:\ind{\Kmc\cup\Phi}\rightarrow 
T_{\Kmc\cup\Phi}$. However, this time, the 
hypothesis is constructed differently by making sure that we do not introduce 
cycles between introduced individuals, so that the flat ABox can always be 
``rolled up'' into a complex ABox without fresh individual names. To be able to 
do this, our introduced individuals are of the form $b_{a,t,i}$, where 
$a\in\ind{\Kmc\cup\Phi}$, $t\in T_{\Kmc\cup\Phi}$, and $i\in\liinv{1}{2^{\lvert 
T_{\Kmc\cup\Phi}\rvert}}$. 
Note that still, there are only exponentially many individuals of this form. 
For convenience, we set $b_{a,s(a),0}=a$. 

The ABox $\Hmc_s$ is then defined as follows, where $T=T_{\Kmc\cup\Phi}$:
\begin{align*}
 \Hmc_s &= \big\{A(b_{a,t,k})\mid a\in\ind{\Kmc\cup\Phi}, 
                             t\in T, 
                             k\in\liinv{0}{2^{\lvert T\rvert}},
                             \\ 
   &\hphantom{= \big\{A(b_{a,t,k})\mid a}
   A\in (t\cap\Sigma)\big\}
      \\
  &\cup\ \{r(b_{a,t_1,k}, b_{a,t_2,k+1}) \mid 
                             a\in\ind{\Kmc\cup\Phi},
                             t_1\in T, 
                             \phantom{\big\}}
                             \\
    &\hphantom{= \big\{A(b_{a,t,k})\mid a}                         r\in\Sigma,
                             t_2\in\sucCand{T}{t_1}{r},
                             k\in\iinv{0}{2^{\lvert T\rvert}-1}\}\\
  &\cup\ \big\{r(a,b)\mid s(b)\in\sucCand{T}{s(a)}{r}\big\}.
  \hphantom{aaaaaaaaaaaaaaaaaa..}
\end{align*}

Since $\Hmc_s$ is of exponential size, we can verify
\mbox{$\Kmc\cup\Hmc\models\Phi$} in exponential time, which means that, since 
we 
have to consider exponentially many selector functions, the entire procedure 
runs in exponential time. One can show similarly as for 
Theorem~\ref{the:compute-flat} that for every selector function $s$ that is 
compatible with $\Kmc$, $\Hmc_s$ 
satisfies~\ref{itm:consistency} and~\ref{itm:signature}. It remains to show that 
for every hypothesis $\Hmc_0$ without fresh individuals of the current abduction 
problem, there is a corresponding compatible selector function $s$ s.t. 
$\Hmc_s$ 
satisfies also~\ref{itm:entailment}, that is, $\Kmc\cup\Hmc_s\models\Phi$. 
For 
this, we first transform $\Hmc_0$ as in 
Lemma~\ref{lem:complex-el-size-upper} into a flat, forest-like hypothesis 
$\Hmc_2$. Recall that this construction assigns every individual $a$ a type 
$\typeI{a^\Imc}{\Imc}$ based on an arbitrary interpretation $\Imc$, and is such 
that, on 
every 
path from a named individual to a fresh individual, no type is repeated. 
Choose~$s$ s.t. $s(a)=\typeI{a^\Imc}{\Imc}$. Using that $s$ is constructed from 
a 
model of $\Kmc$, one can show that $s$ is compatible with $\Kmc$.
We show that also $\Kmc\cup\Hmc_2\models\Phi$ 
iff $\Kmc\cup\Hmc_s\models\Phi$. For this, we define a mapping 
$m:\ind{\Hmc_2}\rightarrow\ind{\Hmc_s}$ inductively as follows:
\begin{itemize}
 \item for every $a\in\ind{\Kmc\cup\Phi}$, $m(a)=a=b_{a,s(a),0}$,
 \item for every $r(a,b)\in\Hmc_2$ s.t. $b\not\in\ind{\Kmc\cup\Phi}$ and 
$m(a)=b_{a,t,i}$, set $m(b)=b_{a,t',i+1}$, where $t'=\typeI{b^\Imc}{\Imc}$. 
\end{itemize}
The induction terminates because $\Hmc_2$ has the ``forest-shape'', that is, 
for every individual name $a\in\ind{\Kmc\cup\Phi}$, the fresh ancestors of $a$ 
form a 
tree. By checking the definition of $\Hmc_s$, one can see that for every 
$A(a)\in\Hmc_2$, we have $A(m(a))\in\Hmc_s$, and for every $r(a,b)\in\Hmc_2$, 
we have $r(m(a),m(b))\in\Hmc_2$. $m$ is thus a homomorphism from $\Hmc_2$ into 
$\Hmc_s$, which can be used as in the proof for 
Theorem~\ref{the:flat-size-upper} to 
show that $\Hmc_2\models\Phi$ implies that also $\Hmc_s\models\Phi$.
\end{proof}

%% file: appendix-complex-solutions-ALC.tex
\section{Complex ABox Abduction for \ALC}

\begin{figure*}
 \begin{framed}
 \begin{align}
  \top& \sqsubseteq\exists r.\top\sqcap\exists s.\top 
  \label{eq:2exp-successors} \\
  \forall (r\cup s).(B\sqcup B')&\sqsubseteq B'
  \label{eq:2exp-transport-bprime}
  \\
  B&\sqsubseteq\Init\sqcap\neg\Bit
  \label{eq:2exp-initialise-counter-starts}
  \\
  \forall (r\cup s).\big(\Init\sqcap\bigsqcup_{i=1}^n\overline{X}_i\big)
  &\sqsubseteq\Init\sqcap \neg\Bit 
  \label{eq:2exp-initialise-counter-ends}
  \\
 \bigsqcapn_{i=1}^{n}\overline{X_i}
&\sqsubseteq\Flip
 \label{eq:2exp-flip-starts}
 \\
 \forall (r\cup s).(\Flip\sqcap\Bit)
 &\sqsubseteq \Flip
  \\
  \bigsqcup_{i=1}^n X_i\sqcap
  \exists (r\cup s).(\neg\Flip\sqcup\neg\Bit)
  &\sqsubseteq \neg\Flip
 \label{eq:2exp-flip-ends}
 \\
  \bigsqcapn_{i=0}^n X_i\sqcap 
  \Flip \sqcap B'
  &\sqsubseteq \Error\sqcup \Goal
  \label{eq:2exp-goal}
  \\
  \Error&\sqsubseteq \neg \Init\sqcap \exists(r\cup s).(\Error\sqcup E_0)
  \label{eq:2exp-error-starts}
  \\
  E_0&\sqsubseteq \bigsqcapn_{i=1}^n\overline{Y_i}\sqcap E \sqcap
  (\NBit\leftrightarrow\Bit) \\
  E\sqcap \NBit&\sqsubseteq\exists (r\cup s).(E\sqcap \NBit)\\
  E\sqcap\neg \NBit&\sqsubseteq\exists(r\cup s).(E\sqcap\neg \NBit)\\
  \bigsqcapn_{i=1}^n Y_i\sqcap E&\sqsubseteq 
  \exists r.E_f\sqcup \exists s.E_f \\
  E_f\sqcap \Flip &\sqsubseteq (\Bit\leftrightarrow \NBit) \\
  E_f\sqcap \neg \Flip &\sqsubseteq (\Bit\leftrightarrow \neg\NBit)
  \label{eq:2exp-error-ends}
 \end{align}
 \end{framed}
 \caption{Part of the background knowledge base requiring a double 
exponentially large hypothesis.}
 \label{fig:lower-size-complex-alc}
\end{figure*}

\TheLowerComplexALC*
  \begin{proof}
  
  The idea is, similar as for the bound in Theorem~\ref{the:el-complex}, to  
  enforce the hypothesis to represent a binary tree with the concept name $B$ 
on 
its leafs. However, this time, instead of using an $n$-bit counter as in 
Theorem~\ref{the:exp-lower-bound-flat}, we use a trick 
from~\cite{CONSERVATIVE_EXTENSIONS} to implement a $2^n$-bit counter. Here, a 
single counter value is encoded by a chain of $2^n$ elements satisfying either 
$\Bit$ or $\neg\Bit$, depending on whether the corresponding bit has the value 
$1$ or $0$. The counting then produces a chain of $2^{2^n}\cdot 2^n$ elements, 
where each consecutive $2^n$ elements represent the next number in the sequence.
 
The hypothesis for our problem will be of the form $C(a)$, where $C$ only uses 
universal restriction, conjunction, and negation on concept names, to express 
that every path of $r$- and $s$-successors must lead to an instance of $B$, 
and reading each path backwards from the $B$-instance to $a$, it must 
encode contain a sequence of $2^{2^n}$ increasing counter values. The 
signature is $\Sigma=\{r,s,\Bit,B\}$ and 
the observation is $\Goal(a)$. To disallow solutions such as $(\forall 
r.\bot\sqcap\forall s.\bot)(a)$, we add $\top\sqsubseteq\exists 
r.\top\sqcap\exists s.\top$ to~$\Kmc$. A concept $B'$ is used to mark elements 
from which 
every path of $s$- and $r$-successors leads to an instance of $B$. 
The concept name $\Bit$ is used for our counter. To identify bit positions in 
the current 
counter value, we use an 
$n$-bit-counter as for Theorem~\ref{the:exp-lower-bound-flat} 
which is initialised at instances of $B$.

  The construction is similar to that given in~\cite{CONSERVATIVE_EXTENSIONS} 
  to give a triple exponential lower bound for witness concepts of conservative 
extensions in $\ALC$, but slightly adapted to be applicable to the abduction 
problem, as well as improved in its representation. 
 The background knowledge base $\Kmc$ contains the axioms in 
Figure~\ref{fig:lower-size-complex-alc}, together with 
Axioms~\eqref{eq:elbot2-init}--\eqref{eq:elbot2-counting-ends}
from Figure~\ref{fig:lower-complex-elbot} as they are, as well as 
Axioms~\eqref{eq:elbot2-counting-starts}--\eqref{eq:elbot2-counting-ends} 
from Figure~\ref{fig:lower-complex-elbot} with $X_i$/$\overline{X}_i$ replaced 
by $Y_i$/$\overline{Y_i}$. In Figure~\ref{fig:lower-size-complex-alc}, we use 
the construct $\forall (r\cup s).C$ as an abbreviation for $\forall 
r.C\sqcap\forall s.C$,  $\exists (r\cup s).C$ as abbreviation for $\exists 
r.C\sqcup\exists s.C$, and $C\leftrightarrow D$ as abbreviation for $(C\sqcap 
D)\sqcup(\neg C\sqcap\neg D)$.

The set of abducibles is $\Sigma=\{s,r,\Bit,B\}$, and our observation is 
$\Phi=\{\Goal(a)\}$. We explain the axioms in $\Kmc$ one after the other.
\begin{itemize}
 \item The axioms taken from Figure~\ref{fig:lower-complex-elbot} implement 
binary 
counters---one using concept names $X_0$, $\ldots$, $X_n$ to represent positive 
bit values, and one using concept names $Y_0$, $\ldots$, $Y_n$ to represent 
positive bit values. We call those two single exponential counters 
\emph{X-counter} and \emph{Y-counter} respectively. The X-counter 
is initialised at instances of $B$, while we do not state yet where the 
Y-counter is initialised. The X-counter is used to keep track of 
the current bit position in the double exponential counter.
 \item \eqref{eq:2exp-successors} ensures that no hypothesis can limit the 
length of $r$ or $s$-paths in an interpretation, and is needed to avoid the use 
of concepts $\forall r.\bot$ and $\forall s.\bot$ in hypotheses.
 \item \eqref{eq:2exp-transport-bprime} uses the concept $B'$ to mark elements 
from which every path along $r$- and $s$-edges leads to an instance of~$B$.
 \item \eqref{eq:2exp-initialise-counter-starts}%
--\eqref{eq:2exp-initialise-counter-ends} make sure that any element from 
which there are only paths of less than $2^n$ elements to an instance of $B$ 
satisfies $\Init$ and $\neg\Bit$. Those first paths of length $2^n$ represent 
the initial value of $0$ of the double exponential counter.
 \item \eqref{eq:2exp-flip-starts}--\eqref{eq:2exp-flip-ends} 
 use the concept name $\Flip$ to mark elements which correspond to a bit that 
should flip for the next double exponential counter value .
 \item \eqref{eq:2exp-goal} now gives the conditions for the observation to be 
entailed. With the previous axioms, it should be clear that $\bigsqcapn_{i=0}^n 
X_i\sqcap\Flip$ is satisfied by instances from which every element on every 
path of length $2^n$ satisfies $\Bit$, or, if we are on the last bit of the 
double exponential counter with the current value of $2^{2^n}-1$. 
Axiom~\eqref{eq:2exp-goal} states now: if we are in such a situation, and we 
additionally satisfy $B'$, which is, every path of $r$- and $s$-successors 
leads eventually to an instance of $B$, then we satisfy either $\Error$ or 
$\Goal$. Consequently, the hypothesis has to make sure that $\Error$ is not 
satisfied in order to entail the observation. Intuitively, $\Error$ can only be 
satisfied if the double exponential counter made an error on one of the paths 
leading to a $B$-instance. Consequently, to avoid $\Error$ from being 
satisfied, all paths to the next $B$-instance have to use the double 
exponential counter correctly, that is, encode all numbers from $0$ to 
$2^{2^n}-1$ in decreasing order.
 \item \eqref{eq:2exp-error-starts}--\eqref{eq:2exp-error-ends} now implement 
this 
error-detection mechanism via the concept $\Error$. Specifically, if we have an 
instance of $\Error$, then there must be a path to some instance of $E_0$ 
before we reached the initial value of the double exponential counter (marked 
using the concept name \Init). $E_0$ marks the first bit where we observe the 
error: it initialises the Y-counter (which uses the 
concept names $Y_i$/$\overline{Y_i}$), and stores the current bit value using 
in the concept name $\NBit$. The value of $\NBit$ is now transported, together 
with the concept name $E$, along some path of $r$- and $s$-successors of length 
$2^n$ until the Y-counter reaches its maximum value. We have then reached the 
corresponding bit position in the previous value of the double exponential 
counter at which the error is supposed to happen, which we mark with the 
concept $E_f$. Now we compare the stored bit-value in $\NBit$ with the current 
$\Bit$-value and verify that a counting error has indeed happened.
\end{itemize}
For a hypothesis to entail $\Goal(a)$, it has to be of the form $C(a)$, where 
$C$ expresses that every path along $r$- and $s$-successors eventually leads to 
an instance of $B$, and that, when following those paths down from these 
$B$-instances to $a$, they encode the sequence of all consecutive $2^n$-bit 
values via the concept $\Bit$. This can only be done by a concept of triple 
exponential size. Specifically, let $\pi$ be a string of $0$s 
and $1$ that is the result of concatenating all $2^n$-bit numbers, starting 
from the lowest value, and ending with the largest, and let $\pi[i]$ denote to 
the $i$th bit on that sequence. Note that $\pi$ has length $\ell=2^{2^n}\cdot 
2^n$. Define $\Bit[i]$ by $\Bit[i]=\Bit$ if $\pi[i]=1$ and $\Bit[i]=\neg\Bit$ 
if $\pi[i]=0$. A hypothesis for $\Goal(a)$ would be the concept $C(a)$ 
recursively defined as follows:
\begin{itemize}
 \item $C_0=B\sqcap\Bit[0]$
 \item for $i\in\iinv{1}{\ell-1}$, $C_i=\Bit[i]\sqcap \forall(r\cup 
s).C_{i-1}$
 \item $C=C_{\ell}$.
\end{itemize}
This finishes the proof.
\end{proof}

\LemABoxToIntAbs*
\begin{proof}
\newcommand{\AllModels}{\mathbf{I}}
 Let $\AllModels$ be the set of all models of $\Kmc\cup\Hmc$.
 We first construct an interpretation abstraction $\IntAbs_1$ based on 
$\AllModels$, which is then modified to become $\Sigma$-complete. 
$\IntAbs_1=\tup{V,\lambda_1,s,\Rf_1,F}$ is constructed together with a 
function $I$ that assigns to every $v\in V$ a set $I(v)$ of pairs 
$\tup{\Imc,d}$ of an interpretation $\Imc\in\AllModels$ and a domain element 
$d\in\Delta^\Imc$. 

For every 
individual name $a\in\ind{\Kmc\cup\Hmc}$, we add a node~$v_a$ with $s(a)=v_a$ 
and
$\lambda_1(v_a)=\{\typeI{a^\Imc}{\Imc}\mid \Imc\in\AllModels\}$. 
We set 
$I(v)=\{\tup{\Imc,a^\Imc}\mid\Imc\in\mathfrak{I}\}$, and
to capture the role assertions, for every $r(a,b)\in\Kmc\cup\Amc$ and 
$\Imc\in\textbf{I}$,
 we add the role assignment
$\tup{v_a,\typeI{a^\Imc}{\Imc},r,v_b}$ to $\Rf_1$. 

We add further outgoing nodes based on the following inductive construction. 
For two connected nodes $v_1$, $v_2$ we denote their \emph{distance} as 
the number of nodes occurring in the shortest path connecting them.
Let $n$ be the maximal nesting depth in $\Amc$, by which we mean nesting of 
existential or universal role restrictions. 
For every node $v$ whose distance to the next named node is 
less than $n$, every type $t\in\lambda_1(v)$, 
and every $\exists r.C\in t$, we add 
a node $v'$
unique to $\tup{v,t,\exists r.C}$ and set
\begin{align*}
 I(v')=&\{\tup{\Imc,d}\mid \tup{\Imc,d_1}\in I(v), \typeI{d_1}{\Imc}=t, \\
 &\hphantom{\{\tup{\Imc,d}\mid }
 \tup{d_1,d}\in r^\Imc, d\in C^\Imc\} \\
 \lambda_1(v')=&\{\typeI{d}{\Imc}\mid \tup{\Imc,d}\in I(v')\}\\
 \tup{v,t,r,v'}\in& \Rf_1.
\end{align*}
Every outgoing node $v$ whose distance to the next named individual is $n$ is 
added to $F$.

The construction directly ensures that $\IntAbs_1$ satisfies 
Conditions~\ref{d:acyclic} and~\ref{d:unique-roles}. It is also not hard to 
understand that $\IntAbs_1$ abstracts at least one interpretation from 
$\AllModels$, and thus a model of $\Kmc$. We can show that it also explains 
$\Phi$, that is, that every model of $\Kmc$ that is abstracted by $\IntAbs_1$ 
is also a model of $\Phi$. 
Specifically, let $\Imc$ be such a model of $\Kmc$, and $h:\Delta'\rightarrow 
V$ be the mapping that witnesses this. Since $\Kmc\cup\Hmc\models\Phi$, for 
every $C(a)\in\Phi$, we must have $a^\Jmc\in C^\Jmc$ for all 
$\Jmc\in\AllModels$, and thus $C\in t$ for every $t\in\lambda_1(s(a))$. By 
Condition~\ref{i:individuals} and~\ref{i:type}, it follows that $\Imc\models 
C(a)$ for every $C(a)\in\Phi$. Now assume $r(a,b)\in\Phi$. We then must have 
$r(a,b)\in\Kmc\cup\Hmc$. If $r(a,b)\in\Kmc$, $\Imc\models r(a,b)$ follows from 
the fact that $\Imc$ is a model of $\Kmc$. Otherwise, $r(a,b)\in\Hmc$, and our 
construction explicitly added $\tup{v_a,t,r,v_b}$ to $\Rf_1$ for every
$t\in\lambda_1(v_a)$, so that by Conditions~\ref{i:individuals} 
and~\ref{i:roles}, also $\Imc\models r(a,b)$. It follows that $\Imc\models\Phi$, 
and thus that $\IntAbs_1$ explains $\Imc$. 

It remains to ensure $\Sigma$-completeness. For this, we modify $\IntAbs_1$ 
into an interpretation abstraction $\IntAbs$ that is $\Sigma$-complete, while 
still abstracting the same models of $\Kmc$. This way, we make sure that 
$\IntAbs$ also explains $\Phi$.
Specifically, $\IntAbs=\tup{V,\lambda,s,\Rf,F}$ is obtained from 
$\IntAbs_1=\tup{V,\lambda_1,s,\Rf_1,F}$ by setting:
\begin{itemize}
 \item for every $v\in V$, $\lambda(v)=\{t\in T_{\Kmc\cup\Phi}\mid 
 t'\in\lambda_1(v)$, $t\cap\Sigma=t'\cap\Sigma\}$,
 \item $\Rf=\Rf_1\cup\{\tup{v_1,t,r,v_2}\mid 
\tup{v_1,t',r,v_2}\in \Rf_1$, $t\in T_{\Kmc\cup\Phi}$, 
$t\cap\Sigma=t'\cap\Sigma\}$
\end{itemize}
Since we only added types, every model abstracted by $\IntAbs_1$ is also 
abstracted by $\IntAbs$. We show that also every model of $\Kmc$ that is 
abstracted by $\IntAbs$ is abstracted by $\IntAbs_1$. 
Assume $\Imc$ is a model of $\Kmc$ that is abstracted 
by $\IntAbs$ via $h:\Delta'\rightarrow V$. 
We show that $\Imc$ then is a model of $\Hmc$, and thus $\Imc\in\AllModels$, 
from which it follows from our construction that for every $d\in\Delta'$, 
$\typeI{d}{\Imc}\in\lambda_1(h(d))$, and thus that $\Imc$ is also abstracted by 
$\IntAbs_1$. For $r(a,b)\in\Hmc$, by Conditions~\ref{i:individuals} 
and~\ref{i:roles}, we must have $\tup{a^\Imc,b^\Imc}\in r^\Imc$, and thus 
$\Imc\models r(a,b)$. Consider $C(a)\in\Hmc$.

\todo[inline]{Notation $d^{\Imc_1}$ not cool and should be improved.}
By construction, for every $d\in\Delta'$, there exists some type 
$t'\in\lambda_1(h(d))$ s.t. $t'\cap\Sigma=\typeI{d}{\Imc}\cap\Sigma$. 
Define $f:\Delta'\rightarrow T_{\Kmc\cup\Phi}$ by setting $f(d)=t'$ for some 
such type. Inspection of the construction of $\IntAbs_1$ shows that there is an 
interpretation $\Imc_1\in\AllModels$ that implements exactly those types, that 
is, for every $d\in\Delta'$, there exists some $d^{\Imc_1}\in\Delta^{\Imc_1}$ 
s.t. $\tup{\Imc,d^{\Imc_1}}\in I(h(d))$ and $\type{d^{\Imc_1}}{\Imc_1}=f(d)$. 
Essentially, the tuples $\tup{d,d^{\Imc_1}}$ are the basis of a 
\emph{$\Sigma$-bisimulation} as in~\cite{FOUNDATIONS_UI} between $\Imc$ and 
$\Imc_1$, from which it follows that for every $C\in\sub{\Hmc}$ and 
$a\in\ind{\Hmc}$, $a^\Imc\in C^\Imc$ iff $a^{\Imc_1}\in C^{\Imc_1}$. We show 
this property directly without introducing the notion of bisimulation first. 

\newcommand{\dist}{\textnormal{dist}}
\newcommand{\nd}{\textnormal{nd}}

For a node $v\in V$, denote by $\dist(v)$ the distance to the next internal 
node (or $0$ if $v$ is internal). 
and for a concept $C$, by $\nd(C)$ the nesting depth of role restrictions of 
$C$. We extend $\dist$ to $\Delta'$ by setting $\dist(d)=\dist(h(d))$.
Recall that $n$ is the maximal nesting depth of any concept in $\Hmc$, and 
that our instruction insures that for any outgoing node, $\dist(v)\leq n$. 
Consequently, every path of of outgoing nodes has at most length $n-\dist(v)$. 
We show by structural induction that for every $C\in\sub{\Hmc}$ and 
$d\in\Delta'$, s.t. $(n-\dist(d))\leq\nd(C)$, 
$d\in C^\Imc$ iff $d^{\Imc_1}\in C^{\Imc_1}$.
We distinguish the 
cases based on the syntactical shape of $C$.
\begin{itemize}
 \item For $C=\top$, the claim follows trivially. 
 \item For $C=A\in\NC$, we note that by construction, 
$\typeI{d}{\Imc}\cap\Sigma=\typeI{d^{\Imc_1}}{\Imc_1}\cap\Sigma$, which implies 
that $d\in A^\Imc$ iff $d^{\Imc_1}\in A^{\Imc_1}$.
 \item For $C=\neg D$, by inductive hypothesis, $d\in D^\Imc$ iff 
$d^{\Imc_1}\in D^{\Imc_1}$, so that also $d\in(\neg D)^\Imc$ iff 
$d^{\Imc}\in(\neg D)^{\Imc_1}$.
 \item Let $C=C_1\sqcap C_2$. We have $d\in(C_1\sqcap C_2)^\Imc$ iff $d\in 
C_1^\Imc$ and $d\in C_2^\Imc$. By inductive hypothesis, this holds iff 
$d^{\Imc_1}\in C_1^{\Imc_1}$ and $d^{\Imc_1}\in C_2^{\Imc_2}$, which holds iff 
$d^{\Imc_1}\in (C_1\sqcap C_2)^{\Imc_1}$.
 \item Finally, let $C=\exists r.D$. 
If $d\in C^{\Imc}$, there must be $\tup{d,e}\in r^{\Imc}$, which 
by~\ref{i:roles} means that $\tup{h(d),\typeI{d}{\Imc},r,h(e)}\in \Rf$. By 
construction, we have $\tup{h(d),f(d),r,h(e)}\in \Rf_1$.
which means by construction of $\IntAbs$ that we have 
$\tup{d^{\Imc_1},e^{\Imc_1}}\in r^{\Imc_1}$. Note that $\dist(e)=\dist(d)+1$, 
and $\nd(D)=\nd(C)-1$, so that we can apply the inductive hypothesis. We obtain
$e^{\Imc_1}\in D^{\Imc_1}$, and thus $d^{\Imc_1}\in(\exists r.D^{\Imc_1})$. 

Now assume $d^{\Imc_1}\in C^{\Imc_1}$. By construction, we then must have 
$\tup{h(d),\typeI{d^{\Imc_1}}{\Imc_1},r,v'}\in\Rf_1$ s.t. $e\in C^{\Imc_1}$ for 
all 
$\tup{\Imc_1,e}\in I(v')$, and by construction, we have 
$\tup{h(d),\typeI{d}{\Imc},r,v'}\in\Rf$. By \ref{i:roles}, there must then be 
$e\in\Delta'$ s.t. $h(e)=v'$. Since thus $\tup{\Imc_1,e^{\Imc_1}}\in I(v')$, we 
obtain that $e^{\Imc_1}\in D^{\Imc_1}$, which by inductive hypothesis means that 
$e^\Imc\in D^\Imc$. It follows that $d\in C^\Imc$.
\end{itemize}
It follows now in particular that for any $a\in\ind{\Hmc\cup\Kmc}$ and 
$C\in\Hmc$, $a^{\Imc_1}\in C^{\Imc_1}$ iff $a^{\Imc}\in C^\Imc$. Since 
$\Imc_1\in\AllModels$, we have $\Imc_1\models\Hmc$, and thus for all 
$C(a)\in\Hmc$, $\Imc_1\models C(a)$. As a 
consequence, we obtain $\Imc\models C(a)$, and thus $\Imc\in\IntAbs$ and $\Imc$ 
is abstracted also by $\IntAbs_1$.
\end{proof}

\LemIntAbsSize*
\begin{proof}
To show that $\IntAbs$ can be reduced into an 
interpretation abstraction of the required size, we use the facts that
\begin{enumerate}[label=\textbf{n\arabic*}]
 \item\label{en:unconnected} nodes that have no connection to a named 
node can be removed,
 \item\label{en:internal} we can identify internal nodes with the same label,
  \item\label{en:branch} if there are two role assignments $\tup{v_1,t,r,v_2}$, 
$\tup{v_1,t,r,v_2'}\in \Rf$ and $\lambda(v_2)=\lambda(v_2')$, then the latter 
can be removed, and
 \item\label{en:length} the length of any path from or to a named node 
that 
does not contain another named node can be restricted to 
$\ell$.%

\end{enumerate}

Point~\ref{en:unconnected} does not require further explanation, while the 
others need more of an argument. 

For~\ref{en:internal}, assume $v_1,v_2\in V$ are two internal nodes s.t. 
$\lambda(v_1)=\lambda(v_2)$. Define a mapping 
$m:V\rightarrow(V\setminus\{v_2\})$ by setting $m(v)=v$ for $v\neq v_2$ and 
$m_1(v_2)=v_1$. We construct a new interpretation abstraction 
$\IntAbs_1=\tup{V_1,\lambda_1,s,\Rf_1,F_1}$ by setting
\begin{itemize}
 \item $V_1=V\setminus\{v_2\}$,
 \item $\lambda_1(v)=\lambda(v)$ for all $v\in V_1$,
 \item $\Rf_1=\{\tup{m(v),t,r,m(v')}\mid \tup{v,t,r,v_1}\in \Rf\}$, and
 \item $F_1=F\cap V_1$.
\end{itemize}
One easily verifies that $\IntAbs_1$ satisfies 
Conditions~\ref{d:acyclic}--\ref{d:signature-edges}, and thus is is 
$\ALC$-conform and $\Sigma$-complete. To show that it explains $\Phi$, we 
need to show that $\IntAbs_1$ abstracts at least one model of $\Kmc$, and every 
model of $\Kmc$ abstracted by $\IntAbs_1$ entails $\Hmc$. Let $\Imc$ be a model 
of $\Kmc$ abstracted by $\IntAbs$ via $h:\Delta'\rightarrow V$. Let $a_{v_1}$ be 
such that $s(a_{v_1})=v_1$, and $a_{v_2}$ 
be such that $s(a_{v_2})=v_2$. We construct a new interpretation $\Imc_1$ which 
has as domain $\Delta^{\Imc_1}=\{d\in\Delta^\Imc\mid h(d)\neq v_2\}$.
Define a mapping 
$m:\Delta^\Imc\rightarrow\Delta^{\Imc_1}$ by setting $m(d)=d$ if 
$d\not\in\Delta'$ or $h(d)\neq v_2$, and $m(d)=a_{v_1}^\Imc$ otherwise.
Construct $\Imc_1$ as 
\begin{itemize}
 \item $A^{\Imc_1}=A^\Imc\cap\Delta^{\Imc_1}$ for all $A\in\NC$, and
 \item $r^{\Imc_1}=\{\tup{m(d),m(e)}\mid \tup{d,e}\in r^\Imc\}$ for all 
$r\in\NR$.
\end{itemize}
The transformation is type-preserving, and since $\Imc\models\Kmc$, so does 
$\Imc_1$. We can obtain a function 
$h':(\Delta'\cap\Delta^{\Imc_1})\rightarrow V'$ witnessing abstraction by 
$\IntAbs_1$ by setting $h'(d)=h(m(d))$, for which 
Conditions~\ref{i:individuals}--\ref{i:roles} are easily verified. Consequently, 
$\IntAbs_1$ abstracts a model of $\Kmc$. Now let $\Imc$ be a model 
of $\Kmc$ abstracted by $\IntAbs_1$. By duplicating the individual $d$ s.t. 
$h(d)=v_1$ to an individual $d'$ with $h(d')=v_2$, we can in a similar fashion 
construct a $\Imc_1$ that is abstracted by $\IntAbs$, and that entails 
$\Kmc\cup\Phi$ iff so does $\Imc$. 

It follows that every model of $\Kmc$ abstracted by $\IntAbs_1$ 
entails $\Hmc$, so that $\IntAbs_1$ explains $\Phi$.~\ref{en:branch} can be 
shown in the same way.

For Point~\ref{en:length}, we only 
consider the case of a path from an named node. The other case is 
similar. Assume there is such a path
\[
 v_1, t_1, r_1, v_2, t_2, r_2, \ldots, r_{n-1}, t_{n-1} v_n
\]
where $n>\ell$. We show that $\IntAbs$ can then be modified in a way that 
preserves all properties required by the Lemma so that this path becomes 
shorter. We ignore the types for now, and consider $\pi=v_1,r_1$, 
$\ldots$, $r_{n-1}$, $v_n$. 
We 
say that an interpretation 
$\Imc$ abstracted by $\IntAbs$ via $h$ \emph{implements $\pi$} iff 
there exists $d_i\in\Delta'$ with $h(d_i)=v_i$ for all $i\in\iinv{1}{n}$ s.t.
$\tup{d_i,d_{i+1}}\in r_i$ for all $i\in\iinv{1}{n-1}$. If there is no 
interpretation $\Imc$ that implements $\pi$, let $i$ be the largest element in 
$\iinv{1}{n}$ s.t. the path $v_1$, $r_1$, $\ldots$, $r_{i-1}$, $v_i$ is 
implemented. Clearly, we can then remove all role assignments 
$\tup{v_i,t,r,v_{i+1}}$, since they do not affect the interpretations 
abstracted by $\IntAbs$. Otherwise, assume the interpretation $\Imc$ implements 
$\pi$, and $h$ is the function by which $\Imc$ is abstracted by $\IntAbs$.
For $i\in\iinv{1}{n}$, let $d_i\in\Delta'$ be s.t. $h(d_i)=v_i$. Because 
of~\ref{i:roles} in Definition~\ref{def:interpretation-abstraction}, for every 
$i\in\iinv{1}{n-1}$ there exists $\tup{v_i,\typeI{d_i}{\Imc}, r_i, v_{i+1}}\in 
\Rf$, so that we can assume that $t_i=\typeI{d_i}{\Imc}$.
Because there are at most $\ell$ pairs $\tup{t,T}$ of a type a set of 
types, by the pigeon hole principle, there must be $i,j\in\iinv{1}{n}$ s.t. 
$i<j$, $\lambda(v_i)=\lambda(v_j)$ and $t_i=t_j$. We modify $\Rf$ so that it 
``skips'' the elements between $v_i$ and $v_j$, by replacing 
$\tup{v_{i-1},t_{i-1},r_{i-1},v_i}$ with $\tup{v_{i-1},t_{i-1},r_{i-1},v_{j}}$.

Denote by $\IntAbs_1=\tup{V_1,\lambda_1,s,\mathfrak{R}_1}$ the result of 
applying 
this operation. %
$\IntAbs_1$ is $\ALC$-conform 
and $\Sigma$-complete if 
so is $\IntAbs$. We show that furthermore, $\IntAbs_1$ still explains $\Phi$, 
that is, that for every 
model $\Imc_1$ of $\Kmc$ is abstracted by $\IntAbs$, we have 
$\Imc_1\models\Phi$. 
\todo[inline]{Not quite: also has to have a model with $\Kmc$!}
Let $\Imc_1$ be a model of $\Kmc$ that 
is abstracted by $\IntAbs_1$ via the mapping 
$h_1:\Delta'\rightarrow V_1$. If $\Imc_1$ does not use the new role 
assigment $\tup{v_{i-1},t_{i-1},r_{i-1},v_{j}}$, then it is also abstracted by 
$\IntAbs$. Otherwise, there are $d_{i-1}',d_{j}'\in\Delta'$ s.t. 
$h_1(d_{i-1}')=v_{i-1}$, $h_1(d_{j}')=v_j$ and $\tup{d_{i-1}',d_j'}\in r_{i-1}$
Based on $\Imc_1$ and the model $\Imc$ used before, we construct a model 
$\Imc_2$ that is abstracted by 
$\IntAbs$ s.t. $\Imc_2\models\Kmc$ and $\Imc_2\not\models\Phi$.  

Intuitively, we adapt $\Imc_1$ %
by putting the path connecting $d_i$ and $d_j$ in $\Imc$ between $d_{i-1}$ and 
$d_i$.
Specifically, we define $\Imc_2$ as follows:
\begin{align*}
 \Delta^{\Imc_2}=&\Delta^{\Imc_1}\cup\Delta^{\Imc}\setminus\{d_j'\}, \\
        A^{\Imc_2}=&A^{\Imc_1}\cup A^{\Imc}
        \text{ for every }A\in\NC,\\
        r^{\Imc_2}=&\{\tup{d_1,d_2}\in r^{\Imc_1}\mid d_1\neq 
            d_j'\text{ or }d_2\neq d_j'\}\\
 &\cup r^{\Imc}\\
 &\cup\{\tup{d_1,d_i}\mid \tup{d_1,d_j'}\in\Imc_1\}\\
 &\cup\{\tup{d_j,d_2}\mid \tup{d_j',d_2}\in\Imc_2\}
  \text{ for every }r\in\NR,\\
        a^{\Imc_2}&=a^{\Imc_1} \text{ for every }a\in\NI .
\end{align*}
Because $\typeI{d_j}{\Imc}=\typeI{d_j'}{\Imc_1}$,
the construction is type-preserving, which implies $\Imc_2\models\Kmc$, but 
also $\Imc_2\not\models\Phi$. Furthermore, $\Imc_2$ is abstracted by $\IntAbs$. 
\end{proof}

\LemIntAbsToABox*
\begin{proof}
For every internal node $v\in V$, let $a_v$ be the individual s.t. $s(a_v)=v$.
We assign to every node $v\in V$ and $i\geq 0$ a 
depth-bounded concept $C_v^{(i)}$, and to every $v\in V$ a concept $C_v$. 
$C_v^{(0)}$ is defined by
\[
    C_v^{(0)}=\bigsqcup\limits_{t\in\lambda(v)}\left(\bigsqcap_{A\in 
t\cap\Sigma\cap\NC}A
    \sqcap\bigsqcap_{A\in (\Sigma\cap\NC)\setminus t}\neg A\right)
\]
For $i>0$, if $v\in F$, $C_v^{(i)}=C_v^{(0)}$, and otherwise
\begin{align}
 C_v^{(i)}&=\bigsqcup\limits_{t\in\lambda(v)}\Big(
 \bigsqcap_{A\in t\cap\Sigma\cap\NC}A
    \sqcap\bigsqcap_{A\in (\Sigma\cap\NC)\setminus t}\neg A
 \\
 &\sqcap\bigsqcap\{\exists r.C_{v'}\mid \tup{v,t,r,v'}\in\mathfrak{R}, v'\text{ 
is outgoing}\}\\
 &\sqcap\bigsqcap_{r\in\NR\cap\Sigma}\forall r.\left(\bigsqcup
 \{C_{v'}^{(i-1)}\mid\tup{v,t,r,v'}\in\mathfrak{R}\}\right) 
\end{align}
We then set $C_v=C_v^{(n)}$, where $n$ is the maximal path length of any path 
of outgoing edges 
As the outgoing nodes cannot form a cycle due~\ref{d:acyclic}, the induction on 
the existential role restrictions always terminates. The depth-bound is used to 
ensure that the induction also terminates for the universal restrictions. These 
universal restrictions are used to restrict the possible successors in 
accordance with~\ref{i:roles}, without conflicting with the internal nodes which 
may form cycles.
 
$\Hmc$ contains the concept assertion $C_v(a_v)$ for every 
internale node $v$, and the role assertion $r(a_v,a_{v'})$ for 
every $\tup{v,t,r,v'}\in\mathfrak{R}$ for which $v$ and $v'$ are internal.
One can verify that $\Hmc$ satisfies the size restrictions given in the lemma.

It remains to show that every model of $\Kmc$ abstracted by $\IntAbs$ is a 
model of $\Hmc$, and that every model of $\Kmc\cup\Hmc$ is abstracted by 
$\IntAbs$.

Assume $\Imc$ is abstracted by $\IntAbs$ via 
$h:\Delta'\rightarrow V$. We 
show by induction on the nesting depth that for every node $v$, 
$h(v)\in C_v^\Imc$: For 
$v\in F$, by~\ref{i:type}, we have $\typeI{h(v)}{\Imc}\in\lambda(v)$, and thus 
\[h(v)\in 
\left(\bigsqcap_{A\in 
t\cap\Sigma\cap\NC}A\sqcap
\bigsqcap_{A\in(\Sigma\cap\NC)\setminus t}
\neg A\right)^\Imc.
\]
For the induction step, assume that for every role successor $v'$ of $v$ we 
have
$h(v')\in C_{v'}^\Imc$. It then follows directly from the definition that then 
$h(v)\in C_{v}^\Imc$. Consequently, $\Imc\models C(a_v)$ for every 
$C(a_v)\in\Hmc$. Let $r(a_v,a_{v'})\in\Hmc$. By 
Condition~\ref{d:unique-roles}, because $v$ and $v'$ are internal, we have 
$\tup{v,t,r,v'}\in\mathfrak{R}$ for every type $t\in\lambda(v)$. It follows 
by~\ref{i:roles} that $\tup{h(v),h(v')}\in r^\Imc$, and thus $\Imc\models 
r(a_v,a_{v'})$. We have shown that $\Imc\models\Hmc$.

Now assume $\Imc\models\Hmc$. We recursively build a function 
$h:\Delta'\rightarrow V$ that satisfies 
conditions~\ref{i:individuals}--\ref{i:roles}, and thus witnesses that $\Imc$ 
is abstracted by $\IntAbs$. For each internal node $v$, we have a named 
individual $a_v$ and we set $h(v)=a_v^\Imc$. Outgoing nodes are assigned 
inductively based on the definition of $C_v$. If $h(d)=v$ and $d\in 
C_v^\Imc$, there is a disjunct $C_{v,t}$ of $C_v$ s.t. $d\in C_{v,t}^\Imc$. If 
$C_{v,t}^\Imc$ has as conjunct
$\exists r.C_{v'}$, then there must be $\tup{d,d'}\in r^\Imc$ s.t. $d'\in 
C_{v'}^\Imc$, and we set $h(d')=v'$. 
As a result, we obtain a mapping $h:\Delta'\rightarrow V$ for some 
$\Delta'\subseteq\Delta^\Imc$ s.t. $h(d)\in C_{h(d)}^{\Imc}$. We show that this 
implies 
$\typeI{d}{\Imc}\in\lambda(h(d))$. Because $d\in C_{h(d)}^{\Imc}$, there exists 
$t\in\lambda(h(d))$ s.t. 
\[d\in%
\left(\bigsqcap_{A\in 
  t\cap\Sigma\cap\NC}A
  \sqcap\bigsqcap_{A\in (\Sigma\cap\NC)\setminus t}\neg A\right)^\Imc, 
\]
which 
means that $\typeI{d}{\Imc}\cap\Sigma=t\cap\Sigma$. By~\ref{d:signature-nodes}, 
we then have $\typeI{d}{\Imc}\in\lambda(h(d))$. We obtain that $h$ 
satisfies~\ref{i:type}. \ref{i:individuals} is satisfied by construction
and~\ref{i:roles} is a consequence of the fact 
that $\typeI{d}{\Imc}\in\lambda(h(d))$ for every $d\in \Delta'$. We have shown 
that $\Imc$ is abstracted by $\IntAbs$.
\end{proof}

\TheDecideComplexALCAbduction*
\begin{proof}
\todo[inline]{Add details?}
\comment{
We first guess a set $V_I$ of internal nodes, together with an assignment 
$\lambda_I:V_I\rightarrow 2^{T_{\Kmc\cup\Phi}}$ and a set of role assignments 
between those nodes. If we are not introducing fresh individuals, then this 
guessing step is replaced by a deterministic double exponential number of 
iterations corresponding to the different possible assignments 
$\ind{\Kmc\cup\Psi}\rightarrow 2^{T_{\Kmc\cup\Phi}}$. Otherwise, we use 
double exponential time on a non-deterministic Turing machine. 

As next step, we generate the outgoing nodes in the following way: for every 
node $v$ that has a distance of less than $\ell$, where $\ell$ is as in 
Lemma~\ref{lem:int-abs-size}, to the next internal node  (including the 
internal nodes), every $t\in\lambda(v)$ and every $\exists r.C\in t$, we add a 
new node $v'$ for which $\lambda(v)$ contains all types $t'\in 
T_{\Kmc\cup\Phi}$ for which there 
exists $t\in\sucCand{T_{\Kmc\cup\Phi}{t}{r}}$ s.t. $C\in t$ and $\Sigma\cup$

to every 
node $v$ and $t\in\lambda(v)$ s.t. $\exists r.C\in\lambda$
}
We guess an $\ALC$-conform, $\Sigma$-complete 
interpretation abstraction $\IntAbs$ within the size bounds 
of Lemma~\ref{lem:int-abs-size}. Because we can reuse outgoing nodes with the 
same label, it suffices to guess one such 
interpretation abstraction that is of at most double exponential size, similar 
as we did for Lemma~\ref{lem:elbot-decide-complex}.
We then verify it 
abstracts some model of $\Kmc$ by guessing for each node $v$ a type 
$t\in\lambda(v)$ and checking whether those types can be implemented by a model 
of $\Kmc$. This check is possible in polynomial time by just considering a 
node, its assigned type, and its successors. 
Finally, we verify that it explains the observation by making 
another such guess to verify whether there exists a model of $\Kmc$ s.t. 
$\Kmc\not\models\Phi$.
    \todo[inline]{Now: case without fresh individual names!}
\end{proof}

%% file: appendix-minimal-solutions.tex
\section{Size-Restricted Abduction}

Regarding the upper bounds, only the argument for flat solutions in $\ALCI$ 
requires some additional argument.
\begin{lemma}
 Existence of hypotheses for size restricted $\ALCI$ abduction problems can be 
decided in $\NExpTime^\NP$.
\end{lemma}
\begin{proof}
 Let $\AbductionProblem=\tup{\Kmc,\Phi,\Sigma,k}$ be the size-restricted 
abduction problem. We modify the $\coNExpTime$-procedure in the proof for 
Theorem~\ref{the:compute-flat} as follows: after we computed in 
deterministic exponential time the set $T_{\Kmc\cup\Phi}$ of types, we guess a 
flat ABox 
$\Hmc$ of size at most $k$ s.t. for every $r(a,b)\in\Phi$, $r(a,b)\in\Hmc$. 
Since $k$ is encoded in binary, $\Hmc$ may be up to exponential in size. To 
decide whether $\Hmc$ is a hypothesis, we use the $\NP$-oracle to decide whether 
for every selector function $s_e:\ind{\Hmc}\rightarrow T_{\Kmc\cup\Phi}$, 
either $s_e$ is not 
compatible with $\Kmc\cup\Hmc$, or for some $C(a)\in\Phi$, $C\not\in s_e(a)$. 
This can be easily verifed by a non-deterministic, polynomially time bounded 
oracle which takes as input the pairs $\tup{T_{\Kmc\cup\Phi},\Hmc}$, guesses 
the function~$s_e$, and accepts if $s_e$ is compatible with $\Kmc\cup\Hmc$ and 
for some 
$C(a)\in\Phi$, $C\not\in s_e(a)$.
\end{proof}

\LemELMinimal*
\begin{proof}
 Let $\phi$, $\Kmc$, $\Phi$ and $\Sigma$ be as in the proof sketch of the main 
text. We need to show that $\tup{\Kmc,\Phi,\Sigma}$ has a hypothesis of size at 
most $2m$ iff $\phi$ is satisfiable.  If $\phi$ is satisfiable by a truth 
assignment $as:\{p_1,\ldots,p_m\}\rightarrow\{0,1\}$, we can build a 
hypothesis of size $2m$ by 
setting
\[
\Hmc=\{\textsf{True}(p_i)\mid a(p_i)=1\}\cup\{\textsf{False}(p_i)\mid 
as(p_i)=0\}.
\]
Now assume $\tup{\Kmc,\Phi,\Sigma}$ has a hypothesis $\Hmc$ of size at most 
$2m$. For every 
variable $p_i$, the only way to entail $P(p_i)$ using only names from $\Sigma$ 
is to add either 
$\textsf{True}(p_i)$ or $\textsf{False}(p_i)$ to the hypothesis. Each such 
assertion adds 2 to the size of $\Hmc$, and 
as $\Hmc$ is not allowed to have a size greater than $2m$, it must thus contain 
exactly one such 
axiom for each variable. The satisfying truth assignment 
$as:\{p_1,\ldots,p_m\}\rightarrow\{0,1\}$ 
is then obtained by setting $as(p_i)=0$ if $\textsf{False}(p_i)\in\Hmc$ and 
$as(p_i)=1$ if 
$\textsf{True}(p_i)\in\Hmc$. Since $\Kmc\cup\Hmc\models C(c_i)$ for every clause 
$c_i$, we can see 
that $as$ makes every clause in $\phi$ true, and is thus a satisfying 
assignment.
\end{proof}

\LemELBotMinimal*
\begin{proof}
We first give a more formal definition of the exponential tiling problem.
  The problem is given by a tuple $\tup{T,T_I,t_e, 
V,H,n}$ of a set $T$ of tile types, a sequence $T_I=t_1\ldots t_m\in T^*$ of 
initial tiles, a 
final tile $t_e$, vertical and horizontal tiling conditions $V,H\subseteq 
T\times T$, and a number $n$ in binary encoding. We call a mapping 
\[
    f:\iinv{1}{2^n}\times\iinv{1}{2^n}\rightarrow T
\]
a \emph{tiling}, which is 
\emph{valid} iff $f(2^n,2^n)=t_e$, for all $i\in\iinv{1}{m}$, $f(i,1)=t_i$, and
for all 
$i\in\iinv{1}{2^n-1}$ and 
$j\in\iinv{1}{2^n}$, 
\[\tup{f(i,j),f(i+1,j)}\in H\]
and 
\[\tup{f(j,i),f(j,i+1)}\in 
V.\]

Based on $T$, $H$, $V$ and $n$, we construct an $\ELbot$ KB $\Kmc$, an ABox 
$\Phi$, and a 
signature $\Sigma$ s.t. $\tup{\Kmc,\Phi,\Sigma}$ has an ABox hypothesis of size 
at most
\[
k=2\cdot (2^{2n}-m)+3\cdot 2\cdot(2^{2n}-2^n) +2
\]
iff the tiling problem has a solution.

We define the signature as 
  \[\Sigma=\{\Start,x,y\}\cup\{A_t\mid t\in 
    T\},\]
where the 
concept name \Start will mark the first tile, the role names $x$ and $y$ 
will denote the 
horizontal and vertical neighbourhood in the tiling, and each tile type $t\in 
T$ 
has a 
corresponding concept name $A_t$. The observation is 
$\Phi=\{\End(a)\}$. 

$\Kmc$ uses 
additional concept names $C_1$, $\overline{C}_1$, $\ldots$, $C_n$, 
$\overline{C}_n$ for 
$C\in\{X,Y\}$ to encode a binary counter for both coordinates, and a special 
concept name $B$ that 
will ensure that every tile is represented in the hypothesis. Together with a 
tile, the concept name 
\Start initialises the counter:
  \[
  \Start\sqcap A_t \sqsubseteq \overline{X}_1\sqcap 
\ldots\sqcap\overline{X}_n
  \sqcap \overline{Y_1}\sqcap\ldots\sqcap\overline{Y_n}
  \quad \text{ for every }t\in T.
  \]
  To entail the observation, both counters must reach $2^n$ on an individual 
name marked with $B$ and some tile type:
  \[
  X_1\sqcap\ldots\sqcap X_n\sqcap
  Y_1\sqcap\ldots\sqcap Y_n\sqcap
  B\sqcap
  A_e
  \sqsubseteq\End %
  \]
  Every domain element represents at most one counter value:
  \[
   C_i\sqcap\overline{C}_i\sqsubseteq\bot \quad\text{ for every 
}i\in\iinv{1}{n}, C\in\{X,Y\}.
   \]
   The counter values are incremented along the $x$ and $y$-predecessors using 
the following 
axioms for $\tup{C,c}\in\{\tup{X,x},\tup{Y,y}\}$:
   \begin{align*}
       \exists c.(\overline{C}_i\sqcap C_{i-1}\sqcap\ldots C_1) &\sqsubseteq 
C_i 
       &&\text{ for }i\in\iinv{1}{n}\\
       \exists c.(C_i\sqcap C_{i-1}\sqcap\ldots\sqcap C_1) &\sqsubseteq 
\overline{C}_i 
       &&\text{ for }i\in\iinv{1}{n}\\
       \exists c.(\overline{C_i}\sqcap\overline{C}_j) &\sqsubseteq 
\overline{C}_i 
       &&\text{ for }i,j\in\iinv{1}{n},j<i\\
       \exists c.(C_i\sqcap\overline{C}_j) &\sqsubseteq Z_i 
       &&\text{ for }i,j\in\iinv{1}{n},j<i.
   \end{align*}
   The vertical and horizontal tiling conditions are ensured using the 
following 
axioms:
   \begin{align*}
     \exists x.A_t\sqcap A_{t'}\sqsubseteq \bot &\text{ for every 
}\tup{t,t'}\in(T\times T)\setminus H \\
     \exists y.A_t\sqcap A_{t'}\sqsubseteq \bot &\text{ for every 
}\tup{t,t'}\in(T\times T)\setminus V
   \end{align*}
The initial tilings from $T_I$ are specified as follows
\begin{align*}
 \Start&\sqsubseteq I_1 \\
 I_i&\sqsubseteq \forall x.I_{i+1} \qquad &&\text{ for }i\in\iinv{1}{m-1}\\
 I_i&\sqsubseteq A_{t_i} \qquad &&\text{ for }i\in\iinv{1}{m}
\end{align*}

   Finally, the special concept name $B$ is used to ensure that every 
hypothesis 
contains an individual name for every tile. This name is initialised by the 
individual satisfying \Start, and then propagated in $x$ and $y$ 
direction, provided that both the $x$ and the $y$ successor satisfy it, or one 
of the counter values is $0$ (in which case we do not require a successor in 
the 
corresponding direction) 
   \begin{align*}
     \Start&\sqsubseteq B \\
     \overline{X}_1\sqcap\ldots\sqcap\overline{X}_n\sqcap \exists y.(A_t\sqcap 
B)&\sqsubseteq B
     \quad &&\text{ for every }t\in T\\
     \overline{Y_1}\sqcap\ldots\sqcap\overline{Y_n}\sqcap \exists x.(A_t\sqcap 
B)&\sqsubseteq B
     \quad &&\text{ for every }t\in T\\
     \exists x.(A_t\sqcap B)\sqcap\exists y.(A_{t'}\sqcap B)&\sqsubseteq B
     \quad &&\text{ for every }t,t'\in T
   \end{align*}
   This completes the construction.

   First, assume the tiling problem has a solution $f$. Based on $f$, we 
construct a hypothesis for $\tup{\Kmc,\Phi,\Sigma}$ of size $k$. We use 
$2^{2n}$ individual names $a_{i,j}$ assigned with coordinates 
$i,j\in\iinv{1}{2^n}$, where for convenience, we let $a_{2^n,2^n}=a$. The 
hypothesis is then:
   \begin{align*}
     \Hmc=&\big(\{A_t(a_{i,j})\mid f(i,j)=t\}
     \setminus\{A_{t_i}(a_{i,1})\mid 1\in\iinv{1}{m}\}\big)
     \\
     & \cup\{x(a_{i+1,j},a_{i,j})\mid i\in\iinv{1}{2^n-1},j\in\iinv{1}{2^n}\}\\
     & \cup\{y(a_{i,j+1},a_{i,j})\mid 
i\in\iinv{1}{2^n},j\in\iinv{1}{2^n-1}\}\\
     & \cup\{\Start(a_{1,1})\}
   \end{align*}
   It is standard to verify that $\size{\Hmc}=k$, $\Kmc\cup\Hmc\models 
B(a_{i,j})$ for every $i,j\in\iinv{1}{2^n}$, as well as $\Kmc\cup\Hmc\models 
\Phi$.

   Now assume $\tup{\Kmc,\Phi,\Sigma}$ has a hypothesis $\Hmc$ of size $k$. 
Since $\Kmc\cup\Hmc\models\Phi$, $\Hmc$ has to ensure the following:
   \begin{enumerate}
    \item $\Kmc\cup\Hmc\models (C_1\sqcap\ldots\sqcap C_n)(a)$ for both 
$C\in\{X,Y\}$. Since 
     $\Start$ is the only name in $\Sigma$ that can trigger the 
assignment of any counter 
     value, there must be some assertion $\Start(a_0)\in\Hmc$, as well 
as a path of length 
     $2^{n+1}$ along $x$ and $y$-successors connecting $a$ to $a_0$.
    \item $\Kmc\cup\Hmc\models B(a)$. While the assertions mentioned so far 
ensure 
     $\Kmc\cup\Hmc\models B(a_0)$, any element on the path connecting $a$ to 
$a_0$ must entail $B$ 
     as well. This means that either 1) their $X$-counter corresponds to $0$ 
and 
they have a 
     $y$-successor satisfying $B$ and $A_t$ for some $t\in T$, 2) their 
$Y$-counter corresponds 
     to $0$ and they have an $x$-successor satisfying $B$ and $A_t$ for some 
$t\in T$, or 3) 
     they have both an $x$-successor and a $y$-successor satisfying $B$ and 
some 
tile.
   \end{enumerate}
   These conditions, together with the fact that every individual can have only 
one counter value 
assigned, require that for every coordinate $\tup{x,y}\in 
\iinv{1}{2^n}\times\iinv{1}{2^n}$ there 
must exist at least one individual satisfying the corresponding counter values, 
for which $\Hmc$ 
explicitly states the tile type and its $x$ and $y$-successor. 

   Except for the $m$ initial tiles, we thus need one concept assertion per 
tile type, plus the ssertion $\Start(a_{1,1})$. Therefore, we need at least 
$2^{2n}-m+1$ concept assertions (for the 
tile types) and 
$2\cdot(2^{2n}-2^n)$ role assertions (for the neighbourhood). Each concept 
assertion has at least size $2$, and every role assertion has size 3, which 
means that 
\[\size{\Hmc}\geq 2\cdot (2^{2n}-m)+3\cdot 2\cdot(2^{2n}-2^n)-2=k.\] 
Since 
also $\size{\Hmc}\leq k$, we obtain that $\size{\Hmc}=k$, and that $\Hmc$ is 
the 
smallest ABox that satisfies all these conditions. We obtain that every 
coordinate $\tup{i,j}\in\iinv{1}{2^n}\times\iinv{1}{2^n}$ is represented by 
exactly one individual name which has a tile type assigned in $\Hmc$. As $\Kmc$ 
furthermore enforces that the tiling conditions are met, we obtain that $\Hmc$ 
can be used to construct a solution to the tiling problem.
\end{proof}
\todo[inline]{Can there be any problem if complex concepts are used?}

\newcommand{\Turing}{\mathfrak{T}}

\LemNewTiling*
\begin{proof}
 We show that the $\NExpTime^\NP$ tiling problem is $\NExpTime^\NP$ hard by a 
reduction from the word acceptance problem for non-deterministic exponentially 
time bounded Turing machines with access to an $\NP$-oracle.

For convenience, we assume the Turing machine $\Turing$ to have a sligthly 
different 
behaviour than usual: while there is a special state $q_?$ to trigger oracle 
queries, we do not use dedicate states for the query outcome. Instead, when in 
query state $q_?$, the machine immediately fails if the oracle $\Turing_o$ 
would accept the 
word, and otherwise goes into another state as specified by the transition 
relation. This is 
without loss of generality, as we can simulate regular oracle calls as follows: 
1) guess the outcome of the oracle call, 2) if the guess is yes, verify this by 
performing the computation in $\Turing$, 3) if the guess is no, verify this 
using the oracle $\Turing_o$. 

Let 
\[
 \Turing_o=\tup{Q_o,\Sigma_o,\Gamma_o,\blank,q_{o0},F_o,\Delta_o}
\]
be the non-deterministic Turing machine to be used by the oracle, where 
\begin{itemize}
 \item $Q_o$ are the states,
 \item $\Sigma_o$ is the input alphabet
 \item $\Gamma_o$ with $\Sigma_o\subseteq\Gamma_o$ is the tape alphabet
 \item $\blank\in(\Gamma_o\setminus\Sigma_o)$ is the blank symbol,
 \item $q_{o0}\in Q_o$ is the initial state,
 \item $F_o$ contain the accepting states, and
 \item $\Delta_o\subseteq Q_o\times\Gamma_o\times 
Q_o\times\Gamma_o\times\{-1,+1\}$ is the transition function, where a tuple 
$\tup{q_1,a_1,q_2,a_2,D}$ states that, when the machine is in state $q_1$ and 
reads an $a_1$ at the current tape position, then it moves to state $q_2$, 
writes $a_2$ at the current tape position, and moves the tape head according to 
the value of $D$.
\end{itemize}
The non-deterministic oracle Turing machine using this oracle is defined as
\[
    \Turing=\tup{Q,\Sigma,\Gamma,\blank,\Gamma_o,q_0,F,\Delta,q_?}
\] 
where
\begin{itemize}
\item $Q$ are the states, 
\item $\Sigma$ is the input alphabet,
\item $\Gamma$ with $\Sigma\subseteq\Gamma$ is the tape alphabet, 
\item $\blank\in(\Gamma\setminus\Sigma)$ is again the blank symbol
\item $\Gamma_o$ is the oracle input tape alphabet,
\item $q_0$ is the initial state, 
\item $F\subseteq Q$ is the set of accepting states, 
\item $\Delta\subseteq ((Q\times\Gamma\times\Gamma_o)\times
(Q\times \Gamma\times\{-1,0,+1\}\times\Gamma_o\times\{-1,0,+1\}\big)
$ is the transition relation, (which works on two tapes, the standard tape and 
the oracle tape), and
\item $q_?\in (Q\setminus F)$ is
the query state to query the oracle.
\end{itemize}

We further assume $\Turing_o$ to be polynomially bounded, that is, there exists 
a polynomial $p_1$ s.t. for input words $w$, each run on $\Turing_o$ uses at 
most $p_1(\lvert w\rvert)$ steps, and $\Turing$ to be 
exponentially bounded, that is, there exists a polynomial $p_2$ s.t. for 
input words $w$, every run of $\Turing$ requires at most $2^{p_2(\lvert 
w\rvert)}$ steps. Note that as a result, we can bound both the tape and the 
length of runs on both Turing machines by $2^{p(\lvert w\rvert)}$ for some 
polynomial~$p$.

Given $\Turing$, $\Turing_o$ and the input word $w=a_1,\ldots,a_n$, we 
construct a $\NExpTime^\NP$-tiling problem. Without loss of generality, we 
assume $n>2$. Our set $T$ of tile types contains the special tiling type $t_e$ 
that also serves as final tile, and all other tile types are pairs 
$t=\tup{t_1,t_2}$, where
\[t_1\in Q\times\Gamma\times\{l,h,r,-,0\},\] 
and 
\[t_2\in Q_o\times\Gamma_o\times\{l,h,r,-,0\}.\] 
The two components of the 
tile 
type are used to define the different tape contents, the first component being 
the tape of $\Turing$, and the second the tape of $\Turing_o$. For a tile type 
$t=\tup{t_1,t_2}$, $t_1=\tup{q,a,D}$ denotes that $\Turing$ is in state $q$, 
the current tape cell stores $a$, and $D$ refers to the relative position to 
the tape head: \emph{l}eft, \emph{h}ere, \emph{r}ight, or elsewhere (-). The 
value of $D=0$ is furthermore used to mark the first configuration. 

The initial tiles are then
\[
 T_I=\tup{t_1^{(1)},t_2^{(1)}}, \ldots \tup{t_1^{(n+1)}, t_2^{(n+1)}}
\]
where the first components are defined by
\begin{itemize}
 \item $t_1^{(1)}=\tup{q_0, a_1, h}$,
 \item $t_1^{(2)}=\tup{q_0, a_2, l}$,
 \item $t_1^{(i)}=\tup{q_0, a_i, 0}$ for $i\in\iinv{3}{n}$, and
 \item $t_1^{(n+1)}=\tup{q_0,\blank,0}$,
\end{itemize}
and the second components are defined by 
\begin{itemize}
 \item $t_2^{(1)}=\tup{q_{o0},\blank,h}$,
 \item $t_2^{(2)}=\tup{q_{o0},\blank,l}$, and
 \item $t_2^{(i)}=\tup{q_{o0},\blank,0}$ for $i\in\iinv{3}{n+1}$.
\end{itemize}
The horizontal tiling condition $H_1$ contains all tuples 
$\tup{\tup{t_1,t_2},\tup{t_1',t_2'}}$, s.t. for $i\in\{1,2\}$, 
$t_i=\tup{q,a,D}$ and $t_i'=\tup{q',a',D'}$:
\begin{itemize}
    \item $q=q'$,
    \item if $D=0$ and $a=\blank$, then $D'=D$ and $a'=\blank$,
    \item if $D=l$, then $D'=h$,
    \item if $D=h$, then $D'=r$,
    \item if $D=r$, then $D'\in\{0,-\}$, and
    \item if $D=-$, then $D'\in\{0,-,l\}$.
\end{itemize}

This makes sure that in the first row of the tiling, all tiles following the 
initial tiles represent the situation where both tapes have a blank symbol, and 
that in every row, the state associated to a tile must be the same.

The vertical tiling condition $V_1$ encodes the transitions from one 
configuration to another. $V_1$ contains all tuples 
$\tup{\tup{t_1,t_2},\tup{t_1',t_2'}}\in V_1$, where for $i\in\{1,2\}$, 
$t_i=\tup{q,a,D}$ and $t_i=\tup{q',a',D}$, we have
\begin{itemize}
 \item $D'\neq 0$,
 \item if $D\neq h$, then $a=a'$,
 \item for $i=2$, $q=q'=q_{o0}$,
 \item for $i=1$, $\Turing$ has a transition from $q$ to $q'$
 \item if $D=h$, then the transition from $q$ to $q'$ reads $a$ and writes $a'$ 
(on the normal tape if $i=1$, and the oracle tape if $i=2$),
 \item if $D=l$, then $D'=h$ if the transition moves the head on the 
corresponding tape to the left, and otherwise \mbox{$D'=-$},
 \item if $D=r$, then $D'=h$ if the transition moves to the right, and 
otherwise $D'=-$, and
 \item if $D=h$, then $D'=r$ if the transition moves to the left, and $D'=l$ if 
the transition moves to the right.
\end{itemize}
In addition, $V_1$ contains $\tup{\tup{t_1,t_2},t_e}$ for every $t_1$ 
with associated state $q_f\in F$, to 
make sure the final tile is reached once the Turing machine enters a final 
state. 

Intuitively, the tiling problem $\tup{H_1,V_1,T_I,t_e,p(n)}$ encodes the 
behavior of $\Turing$ without the oracle. More formally: if we consider the 
variant of $\Turing$ where $q_?$ acts just like a normal state, this 
Turing machine accepts $w$ iff the tiling problem has a solution.

To now take into account the oracle, it remains to specify the tiling 
conditions 
for $H_2$ and $V_2$, which are similar to those for $H_1$ and $V_1$. Note that 
we want the rows of the first tiling \emph{not} to be initial rows of a tiling 
for $\tup{H_2,V_2,f(i),t_e, p(n)}$, unless the encoded state is $q_?$ and 
they encode the inital configuration of an accepting run for $\Turing_o$. Since 
$\Turing_o$ has only one tape, $H_2$ and $V_2$ now only consider the state of 
the first component of the tile type, which has to be $q_?$, and otherwise only 
consider the second component. Regarding the second component, they are defined 
just like $H_1$ and $V_1$ accept that now, they consider transitions in 
$\Turing_o$, and modify the state in the second component and not in the first.
\end{proof}

\input{appendix-minimal-solutions-ALC-lower}

%% file: appendix-minimal-solutions-ALC-lower.tex
\begin{figure*}
 \begin{framed}
  \begin{align}
     \textsf{Start}\sqcap A_t \sqsubseteq
        \bigsqcapn_{i=1}^{n}\big(
            \overline{X}_i\sqcap&\overline{Y_i}\sqcap\overline{Z}_i\big)
  && \text{ for every }t\in T
  \label{eq:tiling2-start}
  \\
  \bigsqcapn_{i=1}^{n}\left(X_i\sqcap Y_i\sqcap Z_i\right)\sqcap
  B_2\sqcap B_4
  &\sqsubseteq\textsf{End} && \text{ for every }t\in T
  \label{eq:tiling2-end}
  \\
  \label{eq:initial-tiles-start} 
  \Start&\sqsubseteq I_1 \\
   I_i&\sqsubseteq \forall x.I_{i+1} \qquad &&\text{ for }i\in\iinv{1}{m-1}\\
   I_i&\sqsubseteq A_{t_i} \qquad &&\text{ for }i\in\iinv{1}{m}
   \label{eq:initial-tiles-end}\\
   C_i\sqcap\overline{C}_i&\sqsubseteq\bot &&\text{ for every }i\in\iinv{1}{n}, 
C\in\{X,Y,Z\}
  \label{eq:tiling2-counting-start}
   \\
    \exists c.(\overline{C}_i\sqcap C_{i-1}\sqcap\ldots\sqcap C_1) 
&\sqsubseteq 
C_i 
    &&\text{ for }i\in\iinv{1}{n}, 
\tup{C,c}\in\{\tup{X,x},\tup{Y,y},\tup{Z,z}\}\\
    \exists c.(C_i\sqcap C_{i-1}\sqcap\ldots\sqcap C_1) &\sqsubseteq 
\overline{C}_i 
    &&\text{ for }i\in\iinv{1}{n}, 
\tup{C,c}\in\{\tup{X,x},\tup{Y,y},\tup{Z,z}\}\\
       \exists c.(\overline{C_i}\sqcap\overline{C}_j) &\sqsubseteq 
\overline{C}_i 
    &&\text{ for }i,j\in\iinv{1}{n},j<i, 
\tup{C,c}\in\{\tup{X,x},\tup{Y,y},\tup{Z,z}\}\\
       \exists c.(C_i\sqcap\overline{C}_j) &\sqsubseteq C_i 
    &&\text{ for }i,j\in\iinv{1}{n},j<i, 
\tup{C,c}\in\{\tup{X,x},\tup{Y,y},\tup{Z,z}\}
  \label{eq:tiling2-counting-end}
  \\
  \label{eq:tiling2-hidden-tiles-start}
     T^*&\sqsubseteq\bigsqcup_{t\in T} A_t^*\\
     A_t^*\sqcap A_{t'}^{*}&\sqsubseteq\bot 
     &&\text{ for every }t,t'\in T\text{ s.t. } t\neq t'\\
     \bigsqcapn_{i=1}^{n}\overline{Z}_i 
     \sqcap \bigsqcup_{t\in T} A_t^*
     &\sqsubseteq \bot \\
     \bigsqcup_{i=1}^{n}Z_i
     \sqcap\bigsqcup_{t\in T} A_t
     &\sqsubseteq \bot 
     \label{eq:tiling2-hidden-tiles-end}
    \\
    \textsf{Start}&\sqsubseteq B_1 
    \label{eq:tiling2-tile-everything-starts}
    \\
     B_1\sqcap (A_t\sqcup A_t^*) &\sqsubseteq B_2&&\text{ for every }t\in T\\ 
     (\bigsqcapn_{i=1}^n\overline{C}_i)\sqcup\exists c.B_2
     &\sqsubseteq  B_c 
     &&\text{ for every }\tup{C,c}\in\{\tup{X,x},\tup{Y,y},\tup{Z,z}\}
     \\
     B_x\sqcap B_y\sqcap B_z &\sqsubseteq B_1 
     \label{eq:tiling2-tile-everything-ends}
     \\
     \exists x.A_t\sqcap A_{t'}&\sqsubseteq \bot 
     &&\text{ for every }\tup{t,t'}\in(T\times T)\setminus H_1 
     \label{eq:tiling2-conditions-h1}
     \\
     \exists y.A_t\sqcap A_{t'}&\sqsubseteq \bot 
     &&\text{ for every }\tup{t,t'}\in(T\times T)\setminus V_1
     \label{eq:tiling2-conditions-v1}
     \\
     (\exists x.A_t\sqcap A_{t'})\sqcup
     (\exists x.A_t^*\sqcap A_{t'}^*)&\sqsubseteq B_3
     &&\text{ for every }\tup{t,t'}\in(T\times T)\setminus H_2 
     \label{eq:tiling2-conditions-h2}
     \\
     \exists z.A_t^*\sqcap A_{t'}^*&\sqsubseteq B_3
     &&\text{ for every }\tup{t,t'}\in(T\times T)\setminus V_1
     \label{eq:tiling2-conditions-v2}\\
     \exists x.B_3\sqcup\exists z.B_3&\sqsubseteq B_3
     \label{eq:tiling2-propagate-error}
     \\
     \bigsqcapn_{i=1}^n\left(X_i\sqcap Z_i\right)\sqcap (B_3\sqcup\neg 
A_{t_{e}})
     &\sqsubseteq B_4
     \label{eq:tiling2-all-tilings-fail}
  \end{align}
 \end{framed}
 \caption{Background knowledge for the reduction from the $\NExpTime^\NP$-tiling 
problem.}
 \label{fig:lower-complex-min-alc}
\end{figure*}

\LemALCMinimal*
\begin{proof}
 The background KB $\Kmc$ can be seen in 
Figure~\ref{fig:lower-complex-min-alc}, and works like it is described in the 
main text.
 The signature of abducibles is \[
\Sigma=\{\Start, x,y,T^*\}\cup\{A_t\mid t\in T\}
                                \]
 and the observation to be explained is $\End(a)$. We describe the different 
 parts of the knowledge base. In the following, we denote by ``base tiling'' 
the tiling on the bottom of the generated cube, which is going to be the 
solution of the tiling problem $\tup{T,T_I,t_e,V_1,H_1,n}$, while we call the 
other tilings $x\times z$-tilings.
 \begin{itemize}
    \item \eqref{eq:tiling2-start} puts the first tile and initialises all 3 
          counters.
    \item \eqref{eq:tiling2-end} gives the conditions for the observation to be 
          entailed: all counters must reach their maximum value, and 
additionally the concepts $B_2$ and $B_4$ must be satisfied. Intuitively, $B_2$ 
signals that every coordinate always has a tile associated, and $B_4$ that for 
none of the rows in the base tiling, a tiling compatible to $H_2$ and 
$V_2$ can be found.
    \item \eqref{eq:initial-tiles-start}--\eqref{eq:initial-tiles-end} make 
sure the initial tiles in $T_I=\{t_1,\ldots, t_m\}$ are placed.
    \item \eqref{eq:tiling2-counting-start}--\eqref{eq:tiling2-counting-end} 
implement the binary counters for the three coordinates.
    \item 
\eqref{eq:tiling2-hidden-tiles-end}--\eqref{eq:tiling2-hidden-tiles-end} 
specify the behaviour of the ``hidden tiles'' that should not be explicit in 
the hypothesis, to allow for different tilings to be tested in the models of 
the hypothesis: the abducible $T^*$ states the existence of a hidden tile, 
only one hidden tile can be used at the same time, and hidden tiles can only be 
used outside of the ground plane of the cube. 
    \item \eqref{eq:tiling2-tile-everything-starts}--%
\eqref{eq:tiling2-tile-everything-ends} 
ensure that every coordinate is assigned a (hidden or explicit) tile type. 
Specifically, $B_1$ is entailed at coordinates for which every preceeding 
coordinate has a tile type associated, $B_2$ is entailed if additionally, the 
current coordinate has a tile type associated, and for $c\in\{x,y,z\}$, $B_c$ 
states the $c$-coordinate has either value 0, or the next lower 
$c$-neighbour satisfies $B_2$. Consequently, we require either $\Start$ or 
$B_x\sqcap B_y\sqcap B_z$ to entail~$B_1$.
    \item \eqref{eq:tiling2-conditions-h1} and~\eqref{eq:tiling2-conditions-v1} 
are 
as in the proof for Lemma~\ref{lem:elbot-minimal}, and make sure that the 
base 
tiling does not break the tiling conditions $H_1$ and $V_1$.
    \item \eqref{eq:tiling2-conditions-h2} and~\eqref{eq:tiling2-conditions-v2} 
test for the tiling conditions on the hidden tilings, and mark errors using the 
concept name $B_3$, which is then propagated towards all $x$- and 
$y$-predecessors through~\eqref{eq:tiling2-propagate-error}.
    \item Finally,~\eqref{eq:tiling2-all-tilings-fail} makes sure that the 
concept name $B_4$---required for the observation---is propagated along the 
$y$-axis to check whether all tiling attempts on the different 
$H_2$-$V_2$-tilings fail.
 \end{itemize}
Similar as for Lemma~\ref{lem:elbot-minimal}, we can find a bound $k$ that 
ensures that for every coordinate at most one individual is used, so that 
hypotheses consists of assertions forming a 
cube, where the bottom side corresponds to a valid tiling, while the 
remaining coordinates have the concept name $T_*$ assigned. The observation 
$\End(a)$ can only be entailed by such a cube if the $\NExpTime^\NP$-tiling 
problem has a solution.
\end{proof}